\documentclass[11pt]{article}
\usepackage[letterpaper]{geometry}
\usepackage[parfill]{parskip}
\usepackage{amsmath,amsthm,amssymb,bbm}
\usepackage{mathtools}
\usepackage{cases}

\usepackage[english]{babel}
\usepackage{dsfont}

\usepackage{microtype}
\usepackage{comment}

\usepackage{subfigure}
\usepackage{algorithm,algorithmic}
\usepackage{color}
\usepackage{appendix}

\usepackage{authblk}

\usepackage{url}
\usepackage[authoryear]{natbib}
\usepackage[colorlinks,citecolor=blue,urlcolor=blue,linkcolor=blue,linktocpage=true]{hyperref}
\usepackage{cleveref}
\crefformat{equation}{(#2#1#3)}
\crefrangeformat{equation}{(#3#1#4) to~(#5#2#6)}
\crefname{equation}{}{}
\Crefname{equation}{}{}

\crefname{definition}{\textbf{definition}}{definitions}
\Crefname{definition}{Definition}{Definitions}
\crefname{assumption}{\textbf{assumption}}{assumptions}
\Crefname{assumption}{Assumption}{Assumptions}

\definecolor{maroon}{RGB}{192,80,77}

\newtheorem{theorem}{Theorem}
\newtheorem{lemma}[theorem]{Lemma}

\newtheorem{example}[theorem]{Example}

\newtheorem{remark}[theorem]{Remark}

\newcommand{\mb}[1]{\boldsymbol{#1}}

\newcommand\independent{\protect\mathpalette{\protect\independenT}{\perp}}
\def\independenT#1#2{\mathrel{\rlap{$#1#2$}\mkern2mu{#1#2}}}

\def\E{\mathbb{E}}
\def\P{\mathbb{P}}

\def\Var{\mathrm{Var}}

\def\tr{\mathrm{tr}}

\def\sign{\mathrm{sign}}

\def\R{\mathbb{R}}
\def\cA{\mathcal{A}}

\def\cC{\mathcal{C}}
\def\cD{\mathcal{D}}

\def\cN{\mathcal{N}}
\def\cP{\mathcal{P}}

\def\cT{\mathcal{T}}
\def\cW{\mathcal{W}}
\def\cX{\mathcal{X}}

\def\cZ{\mathcal{Z}}

\begin{document}

\title{A Minimax Theory for Adaptive Data Analysis}

\author[1,2]{Yu-Xiang Wang}
\author[2]{Jing Lei}
\author[1,2]{Stephen E. Fienberg}
\affil[1]{Machine Learning Department, Carnegie Mellon University}
\affil[2]{Department of Statistics, Carnegie Mellon University}

\maketitle

\begin{abstract}
In adaptive data analysis, the user makes a sequence of queries on the data, where at each step the choice of query may depend on
the results in previous steps.  The releases are often randomized in order to reduce overfitting for such adaptively chosen queries.   
In this paper, we propose a minimax framework for adaptive data analysis. 
Assuming Gaussianity of queries, we establish the first sharp minimax lower bound on the squared error in the order of $O(\frac{\sqrt{k}\sigma^2}{n})$, where $k$ is the number of queries asked, and $\sigma^2/n$ is the ordinary signal-to-noise ratio for a single query. 
Our lower bound is based on the construction of an approximately least favorable adversary who picks a sequence of queries that are most likely to be affected by overfitting.  
This approximately least favorable adversary uses only one level of adaptivity, suggesting that the minimax risk for 1-step adaptivity with $k-1$ initial releases and that for $k$-step adaptivity are on the same order.  The key technical component of the lower bound proof is a reduction to finding the convoluting distribution that optimally obfuscates the sign of a Gaussian signal. 
Our lower bound construction also reveals a transparent and elementary proof of the matching upper bound as an alternative approach to \cite{russo2015controlling}, who used information-theoretic tools to provide the same upper bound. We believe that the proposed framework opens up opportunities to obtain theoretical insights for many other settings of adaptive data analysis, which would extend the idea to more practical realms.

\end{abstract}

\section{Introduction}
In traditional statistical data analysis, the validity of inference requires the models and analyzing protocols to be specified before looking at the data.  In modern scientific and engineering research with large-scale data and complex hypotheses, it is more natural to choose models and inference tasks in a sequential and adaptive manner.  For example, one may want to fit a second model to the data after seeing that the first model did not fit well; or to test significance of the variables chosen by a variable selection procedure. If traditional frequentist inference procedures are applied to these adaptively chosen tasks, the validity are often questionable due to overfitting or what is known as ``Researcher Degree of Freedom'' \citep{leamer1978specification,cawley2010over,simmons2011false}.


In this paper we study the problem of \emph{adaptive data analysis}, where the data analyst makes a sequence of queries to the data, with each query selected adaptively based on the releases of previous queries.
In order to avoid overfitting, the queries are released with additional post-randomization to prevent adaptively selecting queries/hypotheses that overfit the data \citep{dwork2014preserving,dwork2015generalization,bassily2015algorithmic,russo2015controlling}.
The main idea is that if the queries are released in a way such that they provide little information about the details of the dataset, it is unlikely for the subsequent queries to overfit.
A good example is to make the releases differentially private \citep{dwork2014preserving,dwork2015generalization}.
These approaches work directly with information-theoretic quantities,
 therefore are applicable to any query selection procedure so long as it is fed with only sufficiently perturbed releases. Such generality is however obtained at a cost. These methods tend to be overly conservative and it remains unclear what adversary they are really protecting. Also, despite some study in lower bounds \citep{steinke2014interactive,hardt2014preventing}, it is not well understood whether the existing noise-adding procedures are optimal.

We attempt to address these problems in adaptive data analysis. Specifically,
\begin{itemize}
	\item we develop a minimax framework for adaptive data analysis that subsumes some previously studied settings as special cases;
	\item we present the first rate-optimal minimax lower bound for adaptive data analysis in the setting studied in \citet{russo2015controlling};
	\item we present a new proof of the upper bound which yields a sharper constant and a transparent understanding of the least favorable query selection mechanism.
\end{itemize}
The results suggest that, when the query space is rich enough and all queries are Gaussian, independent Gaussian noise adding is minimax optimal up to a constant. We show that for $k$-step adaptive data analysis, the smallest worst-case amplification factor of the squared estimation error that can be achieved by any (possibly adaptive) releasing procedures is $\sqrt{k}$.  Here term ``worst-case'' refers to any possible adaptive query selection mechanism. 
Our motivation here was driven substantially by work on relaxations of the method of differential privacy, and its primary mechanism of protection through additive noise.  We return to that link at the end of Section~\ref{sec:setup}.

\paragraph{Related work}
Adaptive data analysis has been studied in (but not limited to) \citet{dwork2014preserving,dwork2015generalization,bassily2015algorithmic,russo2015controlling,steinke2014interactive,hardt2014preventing}. The most commonly used setting assumes the selection mechanism can adaptively choose any low-sensitivity query. 
The sample complexity is defined as the number of data points needed for $k$-step adaptive data analysis to achieve a target error level with high probability. For $k$-step adaptive data analysis, \citet{dwork2014preserving} produces the first sample complexity upper bound for Laplace noise adding in the order of $\tilde{O}(\sqrt{k}/\epsilon^{2.5})$, where $\epsilon$ is the target error level, defined as the largest absolute error over all $k$ steps. \citet{bassily2015algorithmic} improves the bound to $\tilde{O}(\sqrt{k}/\epsilon^2)$ and extends to approximate differential privacy, as well as convex optimization queries. The factor $\sqrt{k}$ is shown to be optimal  \citep{steinke2014interactive,hardt2014preventing} for polynomial time algorithms or any algorithms if the dimension of $\cX$ is sufficiently large.  But the dependence on $\epsilon$ is suboptimal.

Our results apply to a slightly different setting studied in \citet[Proposition 9]{russo2015controlling}, where the queries are assumed to be jointly Gaussian. This is neither stronger nor weaker than the class of low-sensitivity queries as shown in Figure~\ref{fig:queryclasses}. 
We also define the risk differently as the maximum expected squared error. 
Due to these differences, our bounds are only loosely comparable to those in \citet{dwork2014preserving,bassily2015algorithmic,steinke2014interactive} in terms of the maximum expected absolute error --- a middle ground that both our bounds and theirs imply.
In particular, our upper bound is on the same order as \citet{bassily2015algorithmic} and \citet{russo2015controlling} with a constant improvement over \citet{russo2015controlling} due to a more direct proof. Our lower bound (modulo the differences in settings) $\Omega(k^{1/4}/n^{1/2})$ substantially improves the best available lower bound in \citet{steinke2014interactive}, which translates into $\Omega(\min \{k^{1/2}/n, 1\})$.  Note that both our lower bound and that in \citet{steinke2014interactive} applies to the more general setting of subgaussian queries, but it remains an open problem to find an algorithm that matches the lower bound in this more general regime.


\begin{figure}
	\centering
	\includegraphics[width=0.6\textwidth]{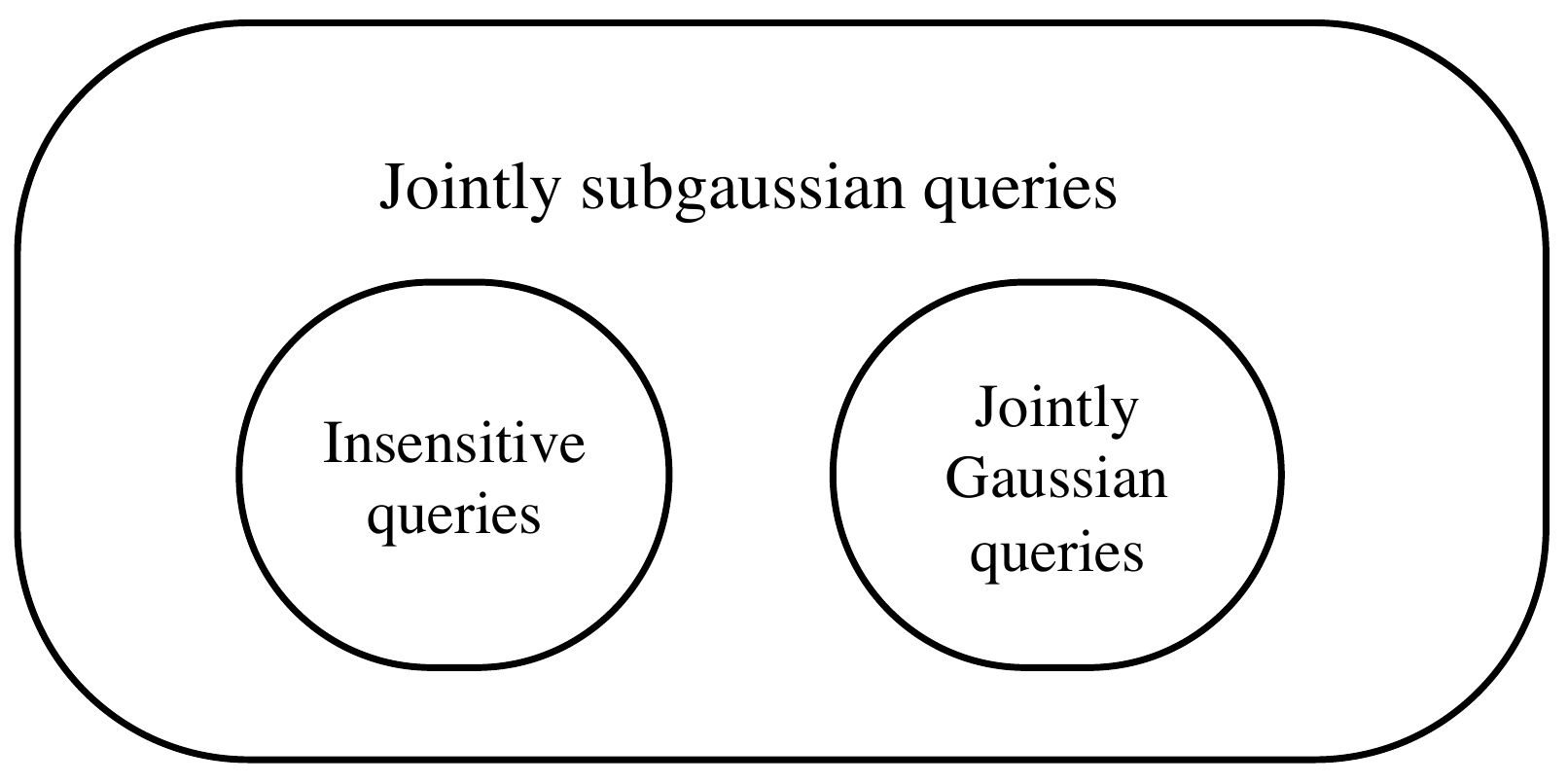}
	\caption{Relationship of the class of queries.}\label{fig:queryclasses}
\end{figure}

When additional assumptions are made, e.g., when we assume $\cX$ is finite, then one can improve the dependence on $k$ exponentially with a computationally inefficient algorithm \citep{dwork2014preserving,bassily2015algorithmic}.

The problem of valid inference for data-dependent tasks has been studied through a different perspective in the statistics community. \citet{lockhart2014significance,taylor2014exact,fithian2014optimal,taylor2015statistical} and others developed a series of ``selective inference'' methods that work with specific variable selection tools (e.g., Lasso) and adjust the confidence intervals or $p$-values accordingly based on the selections such that they have the exact or asymptotically correct frequentist coverage. There are several major differences
between this framework and the adaptive data analysis framework. First, selective inference essentially considers two-step problems, where the variables are selected in the first step, and their significance are tested in the second step.
Second, these methods are \emph{passive observers} in that they release the query without randomization, but live with it and adjust the inference in the second step  to correct the selection bias. 
In other word, the goal of selective inference is to produce valid confidence intervals for even overfitting queries while the goal of adaptive data analysis is to prevent a query that overfits to be selected at all.

\section{A Minimax Framework}\label{sec:setup}

In this section, we describe a general minimax framework for adaptive data analysis, generalizing the problem initially formulated in \citet{dwork2014preserving,hardt2014preventing}.
Given the wide range and vastly different nature of possible query selection mechanisms, we study adaptive data analysis from a conservative approach, where we imagine the queries are selected in a least favorable manner.
Therefore, we consider the adaptive data analysis as a game of two players, a data curator who needs to answer queries accurately with respect to the distribution, and an adversarial query selection mechanism which aims at finding queries that significantly ``overfits'', namely, a function of the data sample that differs a lot from the population value. The procedure is described as follows:

\begin{enumerate}
	\item The query selection mechanism (adversary) picks a query from a collection of queries.
	\item The data curator (player) declares how she is going to release the query value (possibly with randomization).
	\item The query selection mechanism then chooses a new query, adaptively based on the all previous results.
	\item Repeat steps 2,3 until the $k$th query value has been released.
\end{enumerate}
The mathematical problem becomes how the accuracy of query release depend on $k$, when the query selection mechanism behaves adversarially. The same problem can be stated in a different way, when we know that the game will run for $k$ rounds, what is the optimal strategy that the data curator should use and how accurate can the answers be?

Comparing to the literature of differential privacy, this setting is similar to an interactive game where the data curator attempts to protect an attacker from learning about individual information in the data set. The objective is now different, in that we only want to control the generalization error, which is much weaker than differential privacy. We can therefore hope to address the problem for a larger class of problems and practically meaningful data sets where differential privacy is hard to achieve.

\subsection{Notation, terminology, and assumptions}\label{sec:notation}
In the rest of the paper, we simply call the data curator the ``player'', and the query selection mechanism the ``adversary''. We also use $a_{1:i}$ to denote the vector $(a_1,...,a_i)$.

We first define the notations that we will use in this paper, and at the same time formalize what we mean by ``queries'', ``collection of queries'', ``adaptively'' and ``accurately'' in the above description. 
We will pay extra attention to make things explicit in describing what are fixed quantities and what are a random variables. In general, we use capital letters to denote random variables, and use the corresponding lowercase letters to denote the realized values. Many key arguments in this paper will be written as conditional expectations. For random variables $X,Y$, a function $f$, conditional expectation of $f(X)$ given $Y$ is written as $\E_{X|Y} f(X)$, and we will use $\E_{X|y}f(X)$ as a short hand to denote the value $\E_{X|Y=y} f(X)$.
\begin{itemize}
	\item 
	$X$ denotes the data set, a random draw from an unknown distribution $\cD$ defined on $\cX$. The player's task is often to be able to make inferences on $\cD$ but not on a particular realization $X$.
  In order for the game to be meaningful, the player knows $X$ but does not know $\cD$, while the (powerful) adversary knows $\cD$ but does not know $X$.
	\item Let $\cT$ be a class of functions that map each data set to a real number. In other words, each element $t\in\cT$ corresponds to a statistic. We assume that $\cT$ is associated with a natural $\sigma$-algebra.
	Let $$\phi_{\cdot}(\cdot):  \cT \times \cX \rightarrow \R  $$ be the evaluation operator which evaluates $t\in\cT$  on $X$, namely, $\phi_t(X) = t(X)$. Note that $\phi_t(X)$ is a random variable induced by the randomness of $X$.
	\item Denote $T_1,...,T_k$ to be a sequence of random variables defined on $\cT$. Here $\phi_{T_i}(X)$ is a random variable induced by the joint distribution of $X$ and $T_i$. We denote $\phi_{T_i}(X)$ for short by $\phi_{T_i}$.
	\item For every $t\in\cT$, define $\mu_t = \E\phi_t$, so that  $\mu_T = \E\phi_T$ is a random variable induced by $T$.
	\item We denote $\cP_\cT$ and $\cP_\R$ as the class of all probability distributions on $\cT$ and $\R$ respectively.
	\item For $1\le i\le k$, $\cA_i$ denotes the player's release protocol that provides a possibly randomized version of $\phi_{T_i}$. In particular, we allow this release protocol to be different and adaptively chosen for every step. In step $i$, after seeing the $T_i$ chosen by the adversary, the player uses $\cA_i$ to output $A_i \sim \phi_{T_i} + Z_i $ where $Z_i$ is a freshly generated random variable whose distribution $\cZ_i$ is allowed to depend on the entire shared history $H_i = [T_{1:i-1}, A_{1:i-1}, \cZ_{1:i-1}]$ as well as the information only known to the player, e.g., $\phi_{T_{1:i-1}}$ and $Z_{1:i-1}$.
We denote the class of such protocols at step $i$ by $\mathbb{A}_i$. 
	
	
	One should think about $\cA_i$ as an estimator of the population quantity $\mu_{T_i}$ based on $\phi_{T_i}$ yet it does not wish to release much information about the data specific quantity $\phi_{T_i}$. The bigger this class of estimators is, the smaller the minimax risk is.
	Observe that besides causality (it cannot depend on future observations), we place no restrictions on what this $\cA_i$ can do or can depend on. 
	For example, $\cA_i$ could even model the adversary's behavior and predict what the adversary is going to do next and react with $Z_i$ accordingly.

	In our setting, we require that once $\cZ_i \in \cP_\R$ is chosen, this distribution \emph{needs to be declared} to the adversary, before she decides how to select $T_i$. This
  gives the adversary even more power, and leads to more conservative bounds on the total risk.
	
	\item 
	$\cW_i$ is the adversary's selective protocol that selects $T_i\in\cT$ adaptively and possibly randomly. Again, we allow the selective protocol to depend on the history, so that
  $\cW_i$ maps the history and the released distribution $\cZ_i$ to a distribution on $\cT$, from which $T_i$ is sampled.
  $$T_i \sim \cW_i\left(T_{1:i-1},A_{1:i-1},\cZ_{1:i}\right) = \cW_i\left(H_i,\cZ_i\right).$$
We denote the collection of all such protocols as $\mathbb{W}_i$.
  By ``adaptive selection'', we mean that for each $i$, $T_i$ is chosen as a function of everything in the past, and the randomization that $\cA_i$ is going to apply in the current step. 
	$\cW_i$ should be thought of as a powerful adversary that aims at maximizing a given risk function of $\cA_i$. The more powerful this $\cW_i$ can be chosen to do, the bigger the minimax risk. Again, we place no restrictions on what $\cW_i$ can do besides causality.
	\item  We further define loss function $\ell: \R\times \R \rightarrow \R$ to evaluate the generalization risk at each single step, and a conjunction operator $\diamond$ to combine the risks among all $k$ steps.
	The most popular choices of the loss function and the conjunction operator are the square loss: $(A_i-\mu_{T_i})^2$ and the max operator $\diamond(x,y) := \max\{x,y\}$. These will be used to define the minimax risk. Other loss and conjunction operators, e.g., when $\diamond = +$ or $\times$, will be worthwhile exploring and have interesting practical implications.
\end{itemize}
We remark that the joint distribution of random objects $(T_{1:k},A_{1:k},Z_{1:k},\cZ_{1:k},\phi_{T_{1:k}})$
is determined by the distribution of $X$, as well as the protocols $\cA_{1:k}$, $\cW_{1:k}$.  The protocols $\cA_i$
and $\cW_i$ can also depend on previous protocols $\cA_{1:i-1}$ and $\cW_{1:i-1}$, we omit this dependence in our notation for simplicity.  It is important to keep in mind that although the protocols may involve generating fresh random numbers, the protocols themselves are deterministic and can be specified before the game starts.



\subsection{Problem Setup}
We can minimally specify the minimax problem by a triplet $(\cT,\ell,\diamond)$, where $\ell$ is a loss function that measures the performance, $\diamond$ denotes how we can combine expected loss of each round, and $\cT$ is the class of functions that $\cW_i$ can choose from. 

For notational simplicity, we use $H_{i} = [\cZ_{1:i-1},T_{1:i-1},A_{1:i-1}]$ to denote the shared knowledge up to step $i$ by both the player and adversary, and we use $\bar{H}_{i} = [H_i, Z_{1:i-1}]$ to denote the entire history, which further includes information not known to the adversary. Since $T_i \sim \cW_i(H_i, \cZ_i)$,  and $A_i= \phi_{T_i} + Z_i$, where $Z_i\sim \cZ_i(\bar H_i)$. we will just use conditional expectations $\E_{T_i}$ and $\E_{Z_i}$ to denote the expectation with respect to the distribution $T_i\sim \cW_i$ and $Z_i\sim \cZ_i$.

The minimax risk studied in this paper is then
\begin{align*}
 F_k(\mathbb A_{1:k}, \mathbb W_{1:k},\mathcal T,\cD) \inf_{\cA_{1:k}\in \mathbb A_{1:k}}\sup_{\cW_{1:k}\in\mathbb W_{1:k}} \E\ell(A_{1},\mu_{T_1})\diamond ...\diamond
  E\ell(A_{k},\mu_{T_k})
\end{align*}

The minimax problem for $k$-step adaptive data analysis is simply finding the upper and lower bound for $F_k$. 
Note that in the most generic form, we do not wish to constrain the power of adversary and the search space of randomized release algorithm, but there could be special cases when it makes sense to choose $\mathbb{A}_{1:k}$, $\mathbb{W}_{1:k}$ to be their subsets.

In \citet{dwork2014preserving}, $\cT$ is the class of statistical queries that are uniformly bounded. \citet{bassily2015algorithmic} studied the problem where $\cT$ class of functions with sensitivity of changing one data point uniformly bounded, and also optimization queries.  In \citet{russo2015controlling}, $\cT$ is defined as a distribution dependent class of functions such that if $Z\sim \cD$, $(\phi_{t}(X): t\in \cT)$ is a Gaussian process with uniformly bounded variance $\sigma^2$. The loss function $\ell$ is $|A_i -  \mu_{T_i}|$ in all previous studies except in the optimization queries where $\ell$ is the expected excess risk. The conjunction operator $\diamond$ is pointwise maximum in all these previous studies, which requires controlling the worst error among all queries.

In this paper, we will restrict our attention to the problem when $\ell$ is taken to be $|A_i-\mu_{T_i}|^2$ and focus on the case where $\cT$ is a class of functions that satisfies the following assumptions:
\begin{enumerate}
	\item[(A).] For any finite subset of ${t_1,...,t_k}\subset \cT$, $\phi_{t_1,...,t_k}(X)$ is jointly normal. 
	\item[(B).] For every individual $t\in \cT$, $\Var(\phi_{t}(X)) \leq \sigma^2$.  
\end{enumerate}
In other words, $(\phi_{t}:t\in\cT)$ is a certain Gaussian process with bounded variance.
This is the same setting of adaptive data analysis used in \citet{russo2015controlling}.

Under these two assumptions we will provide an upper bound and matching minimax lower bound (up to a small constant factor) of $F_k$ for every integer $k$, which is in the order of $\Theta(\sqrt{k}\sigma^2)$.

Before we present the results, we would like to say a few words about these assumptions. The requirement of Gaussian process is rather strong and typically only holds asymptotically when $\cT$ is a Donsker class (see e.g., \citet{van2000asymptotic}) with respect to $X$'s distribution and $X$ consists of iid draws from a common distribution with sample size $n\rightarrow \infty$. Finite sample results are not hard to derive if a Berry-Esseen type finite sample result holds for all functions that are selected. Still this is weaker than the requirement of differential privacy which places boundedness assumptions on the space $\cX$ and uniform low sensitivity for queries in $\cT$. If a more refined argument can be used to establish comparable results for a class of jointly subgaussian random variables, then this class strictly contains the class that differential privacy is able to handle. The assumption on the marginal variance is quite weak and unavoidable in the minimax sense. 
Finally, in the typical statistical setting where $X$ consists of independent draws from the same distribution, if $\phi_{t}(X)$ corresponds to the sample mean of a function $f$ with uniformly bounded variance over $t$, then $\sigma^2 = O(1/n)$. This is the scaling one should keep in mind when comparing our results to earlier results in \citet{bassily2015algorithmic,russo2015controlling}.


\section{Results}
We first present results for a simpler version of the adaptive data analysis that has only one step adaptivity, where the adversary chooses the queries by explicitly maximizing the selection bias in the form of a conditional expectation. Building upon this result, we extend the argument to form an explicit upper bound for the minimax risk in the $k$-step setting. The proof provides intuition for constructing the minimax lower bound presented in Section \ref{sec:lower_bound}.

\subsection{1-Step Adaptivity}
Our first result applies to the case when $(T_1,...,T_{k-1})$ take an arbitrary fixed vector $(t_1,t_2,...,t_{k-1})$. For each $i = 1,...,k$, we will choose release protocol $\cA_i$ to be such that $A_i=\phi_{t_i}(X) + Z_i$ where $Z_i\sim \cN(0,w^2)$ is a freshly drawn normal random variable.

Then $T_k$ is chosen adaptive by the adversary, after observing the realized values of $A_1,...,A_{k-1}$.  In other words, $T_k$ is sampled from a distribution $\cP$ on $\cT$, with $\cP$ depending on $(t_{1:k-1},A_{1:k-1})$. Again, we emphasize that  $T_k$ belongs to the class $\cT = \{t: | \phi_t(X) \sim \cN(\mu_t,\sigma_t^2), \sigma_t^2\leq \sigma^2\}$.  The key idea is that the choice of $T_k$ boils down to choosing a covariance vector $\Sigma_{k,1:k-1}$ with previously selected $\phi_t,...,\phi_{t_{k-1}}$. The following result constructs the least favorable choice of $T_k$.

We use $\Sigma_{1:i}$ to denote the covariance matrix of $\phi_{t_{1:i}}$, and $\Sigma_{j,1:i}$ the covariance vector
between $\phi_{t_j}$ and $\phi_{t_{1:i}}$.
\begin{theorem}[1-step adaptivity]\label{thm:one-step_adaptive}
	Let $t_1,...,t_{k-1}\in \cT$. Moreover, let observation noise $Z_i \sim \cN(0,w^2)$ and $T_k$ generated by any adaptive selection protocol. Then the squared bias
	$$
	\sup_{t_1,...,t_{k-1}\in \cT} \left[\E (\phi_{T_k}-\mu_{T_k})\right]^2\leq \frac{(k-1)\sigma^4}{w^2}.
	$$
	Moreover, if $w^2 = \sqrt{k-1}\sigma^2$, the square error of the estimate
	$$
	\E (A_k -\mu_{T_k})^2 \leq (2\sqrt{k-1}+1)\sigma^2.
	$$
\end{theorem}
\begin{proof}[Proof sketch]
The proof relies on the law of total expectation that expands the bias into	
\begin{align*}
\E (\phi_{T_k}-\mu_{T_k}) &= \E_{A_{1:k-1}}\E_{T_k|A_{1:k-1}} \E_{\phi_{T_k}|T_k,A_{1:k-1}} (\phi_{T_k}-\mu_{T_k}) \\
&\leq \E_{A_{1:k-1}}\sup_{t_k\in\cT} \E_{\phi_{t_k}|A_{1:k-1}} (\phi_{t_k}-\mu_{t_k}).
\end{align*}
Since $A_{1:k-1} = \phi_{t_{1:k-1}} + Z_{1:k-1}$ and $\phi_{t_k}$ are jointly normal, we can explicitly write down the conditional expectation 
$$
\E(\phi_{t_k}-\mu_{t_k}\mid \phi_{1:k-1})=\Sigma_{k,1:k-1}^T (\Sigma_{1:k-1} + w^2I_{k-1})^{-1} (\phi_{t_{1:k-1}}+Z_{1:k-1} -\mu_{t_{1:k-1}}).
$$
Finding the supremum of $t_k\in\cT$ reduces to finding the maximum over the covariance $\E(\phi_{t_k}-\mu_{t_k})(\phi_{t_{1:k-1}}-\mu_{t_{1:k-1}}) = \Sigma_{k,1:k-1}=:v$ and variance $\Var(\phi_{t_k}) =: w^2$.
These quantities cannot be arbitrary, since $w^2\leq \sigma^2$ and the covariance matrix $\Sigma_{1:k}$ need to be positive definite. Under these constraints, this optimization for is a quadratically constrained linear optimization and we can write down optimal solution in closed form. It remains to evaluate the outer most supremum over $t_1,...,t_{k-1}\in \cT$ which is standard calculations.Details are left to the full proof in Appendix \ref{sec:proof_one_step}.
\end{proof}

\begin{remark}[Sharpness]
	The bound in Theorem~\ref{thm:one-step_adaptive} is sharp because if we take $t_1,...,t_k$ such that $\phi_{t_{1:k-1}} \sim \cN(0,\sigma^2 I)$, then expected squared bias $\geq \frac{(k-1)\sigma^4}{w^2 + \sigma^2}$. When $w^2\gg \sigma^2$, this nearly attains the upper bound (up to a multiplicative factor of $\frac{w^2 + \sigma^2}{w^2}$).
\end{remark}

\subsection{$k$-Step Adaptive Data Analysis: Upper Bound}
Now we extend the above argument to $k$-step adaptive data analysis.
\begin{theorem}[upper bound for $k$-step adaptivity]\label{thm:k-step_adaptive}
	Let the distribution of data $X$ and class of functions $\cT$ obey our assumptions. Now let $T_1,...,T_k$ be random variables drawn by any (potentially randomized) adaptive procedure that chooses $T_i\in \cT$ based on outputs of actively perturbed statistics $A_{1} \sim \cN(\phi_{T_1},w_1^2),..., A_{i-1} \sim \cN(\phi_{T_{i-1}},w_{i-1}^2)$. Then
	for any integer $k$, the square bias of $\phi_{T_{i-1}}$ obeys
	$$
	\left|\E \phi_{T_k}  - \mu_{T_k} \right|^2 \leq \sigma^4\sum_{i=1}^{k-1}\left(\frac{1}{w_i^2} + \frac{\sigma^2}{w_i^4}\right)\,,
	$$
	where the expectation is taken over the both the randomness of $X$, the randomness in the adaptive choice of $(T_1,...,T_k)$, and the randomness of perturbation used in $(A_1,...,A_k)$. Furthermore, by taking $w_i^2= \sqrt{k-1}\sigma^2$ for all $i<k$ and $w_k=0$ we have
	$$
	F_k\leq \max_{i=1,..,k}\E\left| A_i - \mu_{T_i}\right|^2 \leq 2(\sqrt{k-1}+1) \sigma^2\,.
	$$
\end{theorem}

The $k$-step adaptive analysis upper bound is on the same order as in Theorem~\ref{thm:one-step_adaptive} where we only allow one step adaptivity.

\begin{proof}[Proof Sketch]
Similar to the previous theorem, we use the law of total expectation to expand the expectation. It is more involved in that we need to expand $\E \left|\phi_{T_k}  - \mu_{T_k} \right|^2$ recursively into
$$
\E_{A_1}\E_{T_{2}|A_{1}}\E_{A_{2}|T_{1:2},A_{1}}...\E_{T_{k-1}|A_{1:k-2}}\E_{A_{k-1}|T_{1:k-1},A_{1:k-2}}\E_{T_k|T_{1:k-1},A_{1:k-1}}\E_{\phi_{T_k}|T_{1:k},A_{1:k-1}} \left|\phi_{T_k}  - \mu_{T_k} \right|^2.
$$
An upper bound can be obtained by replacing all $\E_{T_i|A_{1:i-1}}$ from a specific selection rule $\cW_i$ with a supremum over $\cT$.
By bias and variance decomposition, it can be shown that the dominating term is the squared bias. Using the same argument as in Theorem \ref{thm:one-step_adaptive}, we can write down the conditional bias and maximize it explicitly, which gives us
$\E(\phi_{T_k} - \mu_{T_k})  \leq \sqrt{\sigma^2 \E \mb f_{k-1}}$
where
$$\mb f_{k-1}:= (A_{1:k-1} -\mu_{T_{1:k-1}})^T (\Sigma_{1:k-1}+W_{1:k-1})^{-1}\Sigma_{1:k-1} (\Sigma_{1:k-1}+W_{1:k-1})^{-1}(A_{1:k-1} -\mu_{T_{1:k-1}}).$$
In the above expression, $W_{1:k-1}$ is the (diagonal) covariance matrix of the noise we add in the first $k-1$ iterations.
Note that $\mb f_{i-1}$ can be defined for any $1\leq i\leq k$. 

It turns out that we can ``peel off'' the supremum one at a time from the inner most conditional expectation all the way to the first one like an ``onion'', using the following formula (details in Lemma~\ref{lem:recursive_fk})
$$\sup_{t_{i-1}\in\cT}\E_{A_{i-1}|T_{1:i-2},t_{i-1},A_{1:i-2}} \mb f_{i-1}  \leq \mb f_{i-2} + \frac{\sigma^2}{w_i^2} + \frac{\sigma^4}{w_i^4}.$$
Applying this recursively and summing up the residuals gives us the bound of bias for any $k$. We invite readers to check out the detailed proof in Appendix \ref{sec:proof_k_step_upper}.
\end{proof}

\begin{remark}[Sharpness]
Since $1$-step adaptivity is a special case of $k$-step adaptivity, the choice of $t_1,...,t_{k}$ that nearly attains the risk bound also applies here, when we choose $w_i=w$ for all $i$.
\end{remark}
Note that the bound in Theorem~\ref{thm:k-step_adaptive} is on the same order as the result in \citet[Proposition 9]{russo2015controlling}. Our proof leads to a sharper constant and transparent understanding of the least favorable adaptive selection protocol. Despite the differences in the settings, we remark that our bound is also on the same order as \citet{bassily2015algorithmic}, modulo that \citet{bassily2015algorithmic} also has strong concentration --- a characteristic that follows from McDiarmid's inequality under the low-sensitivity assumption. Our proof technique is arguably more direct, as we do not rely on differential privacy to control the generalization error.

\subsection{Minimax Lower Bound}\label{sec:lower_bound}

In the previous section, we used $\cA_i: = \phi_{T_i}+\cN(0,w_i^2)$, for $w_i^2=\sqrt{k-1}\sigma^2$, and showed that 
$$F_k = O(\sqrt{k-1}\sigma^2).$$
In this section, we will work out a lower bound of $F_k$, which justifies that the independent Gaussian noise adding is minimax rate optimal among all adaptive noise-adding procedures.

We need a richness assumption on the query class $\cT$ to establish our lower bound.
\begin{enumerate}
  \item [(C).] There exist $t_1,...,t_{k-1}$ such that $\Sigma_{1:k-1}=\sigma^2 I$. For any vector $s\in\{\pm 1\}^{k-1}$, there exists a $t_k\in\mathcal T$ such that
  ${\rm Var}(\phi_{t_k})=\sigma^2$, and ${\rm Cov}(\phi_{t_k},\phi_{t_{1:i-k}})=s\sigma^2/\sqrt{k-1}$.
\end{enumerate}
The richness assumption is natural.  If the class $\cT$ is too small, then the adversary is less powerful.
The smallest $\cT$ satisfying Assumption (c) has cardinality at least $(k-1)+2^{k-1}$.
We further discuss condition (C) in Section~\ref{sec:discussion} with an example.
\begin{theorem}[Minimax lower bound]\label{thm:lower_bound}
	Suppose $\cT$ satisfies Assumptions (A), (B) in Section \ref{sec:setup}, and the richness assumption (C).
   Then the minimax risk 
	$ F_k  \geq \frac{\sqrt{k-1}\sigma^2}{2\sqrt{3}}. $
\end{theorem}
\begin{proof}[Proof Sketch]
	The idea of the proof is that we construct a specific selection rule and show that any adaptively chosen randomized release algorithm cannot reduce $F_k$ to smaller than $C\sqrt{k-1}\sigma^2$ for a universal constant $C=1/(2\sqrt{3})$. To this end, we choose the first $k-1$ selection procedure to be completely non-adaptive and the last one to be adaptive. Specifically, set $T_{1:k-1} = t_{1:k-1}$ such that $\Sigma_{1:k-1}=\sigma^2 I_{k-1}$, i.e., the first $k-1$ queries are independent. Then the adversary $\cW_k$ uses the Bayes classifier to predict the signs of $\phi_{t_{i}}-\mu_{t_{i}}$ for each $i$. This is the likelihood ratio test which output $1$ if
	$$
	\log \frac{p(A_i|\phi_{t_i}-\mu_{t_i} > 0)}{p(A_i|\phi_{t_i}-\mu_{t_i} < 0)} \geq 0
	$$
	and $-1$ otherwise, where $p(\cdot\mid\cdot)$ is the conditional density.	Let the predicted sign vector be $\hat{s}$, $\cW_k$ selects $t_k$ such that $\E(\phi_{t_k})=0, \Var(\phi_{t_k})=\sigma^2$ and in addition the covariance vector is chosen to be
	$$\E[ (\phi_{t_k}-\mu_{t_k}) (\phi_{t_{1:k-1}}-\mu_{t_{1:k-1}})]  = \begin{cases}
	\frac{\hat{s}\sigma^2}{\sqrt{k-1}} & \text{ if } \E_{Z_k|\bar{H}} Z_k > 0\\
	-\frac{\hat{s}\sigma^2}{\sqrt{k-1}} & \text{ otherwise.}
	\end{cases}$$
	Intuitively, the adversary tries to maximize the bias in the last step as much as she can. Once $\phi_{t_{1:k-1}}$ is realized, the plan is to choose $\phi_{t_k}$ with maximally allowed correlation with all previously released queries, where the signs of correlations are chosen such that all bias caused by correlation will have the same sign.  This is the ``secret sauce'' that the optimal adversary used in our proof of Theorem~\ref{thm:one-step_adaptive}.

	The problem now reduces to finding the best randomized release algorithm that confuses the estimation of ${\rm sign}(\phi_{t_i}-\mu_{t_i})$ as much as possible. This is a large space to search. Our strategy is to first study the properties of an optimal release algorithm. We first show that the optimal strategy in the last step adds no noise at all. Because of this, we can evaluate the square of Gaussian conditional bias in the last step in a closed form as a function of $\hat{s}$ and $\phi_{t_{i}}$ . Specifically, it is equal to $\frac{\sigma^2}{\sqrt{k-1}}\sum_i \hat{s}_i(\phi_{t_{i}}-\mu_{t_i})$. Observe that if $\hat{s}_i$ is correct for all $i$, then this bias is in the order of $\sqrt{k-1}\sigma$ and if $\hat{s}_i$ is a purely random guess, then the expectation of this quantity is $0$. We furthermore show that in steps $1$ through $(k-1)$ the optimal strategy always adds noise that is $0$-mean. There is some additional technicality to deal with adaptively chosen noise distributions, but the conclusion is that choosing the noise adaptively is not going to be very useful for getting a smaller risk against this given adversary.
	
With these reductions, we are able to formulate the problem of finding the optimal noise to add as a variational convex optimization problem and by strong duality we are able to construct near optimal dual functions that provides a lower bound
	$$\E \hat{s}_i(\phi_{t_{i}}-\mu_{t_i}) \geq \frac{\sigma^2}{\sqrt{3}\sqrt{\E w_i^2}} - \frac{\sigma^4}{2\sqrt{3}(\E w_i^2)^{3/2}}$$
	as a function of the expected\footnote{The expectation is taken over history up to Round $i$ and the adaptive strategy of $Z_i$.} variance of $Z_i$ (details in Lemma~\ref{lem:best_bayes_error_normal}). This bound is sharper for small signal to noise ratio. In fact, we show that at the limit when the signal to noise ratio converges to $0$, the bound is exact and the optimal distribution of $Z_i$ is the uniform distribution. For the region the bound that is not sharp, e.g., when $\E w_i^2 <\sigma^2$, we find a more naive lower bound separately so that the lower bound is a monotonically decreasing function in $\E w_i^2$.
		
	Finally, we balance the error in the first $k-1$ steps and the error in the last step by choosing $\E w_i^2$ appropriately, which ultimately gives a lower bound of $F_k$ as claimed.
\end{proof}

\begin{remark}[$1$-step adaptivity $\approx$ $k$-step adaptivity]
The lower bound construction essentially uses an adversary that gathers as much information as possible in the first $k-1$ step non-adaptively, and then uses only $1$-step adaptivity in the end to maximize the bias. This suggests that a powerful adversary that is adaptive in all rounds is not significantly more harmful than one that is allowed to behave adaptively only in the last round.
\end{remark}

\begin{remark}[Approximate Least Favorable Adversary]
The estimator $\hat s_i$ considered in our construction only uses the model $A_i=\phi_{t_i}+Z_i$ and the marginal distribution of $Z_i$.  However, when an estimate of ${\rm sign}(\phi_{t_i}-\mu_{t_i})$ is needed by the adversary, the whole history $H_{k-1}$ is available. Thus a potentially more accurate estimate would be
$$
{\rm sign}\left(\log\frac{\P(A_{t_i}\mid \phi_{t_i}-\mu_{t_i}>0, H_{k-1})}{\P(A_{t_i}\mid\phi_{t_i}-\mu_{t_i}<0, H_{k-1})}\right)\,.
$$
Our results indicates that a sub-optimal estimate of ${\rm sign}(\phi_{t_{1:k-1}}-\mu_{t_{1:k-1}})$ already gives a rate optimal lower bound.
Therefore, there is little gap between the simple Bayes classifier used by this approximately least favorable advisory and the best classifier that uses the full history.
\end{remark}


\section{Discussion}\label{sec:discussion}


\paragraph{Richness, dimension and Uniform Convergence.}

As we point out in Assumption (C), the lower bound applies for each given $k$ when there is a sufficiently rich class of functions $\cT$ that satisfy Assumption (A) and (B).  So what happens to the minimax risk when the class of queries is not sufficiently rich? 

For instance, if $\cT$ contains only a finite number of functions, or is a class of smooth functions that has slowing growing metric entropy, then standard uniform convergence argument implies that for sufficiently large $k$, the minimax risk will no longer be proportional to $\sqrt{k}$, but rather become some quantity independent of $k$. In addition, these rates can be achieved without randomization, which essentially deems the whole discussion of adaptive data analysis meaningless.

This picture is more intricate than just the two extremes. For any finite $k$, there is often a large space between small function classes that has uniform convergence (Uniform Glivenko-Cantelli), and big function classes with minimax risk growing in the order of $O(\sqrt{k})$. For instance, adaptive data analysis can be meaningful even for a finite class of functions, if $\sqrt{k}\ll \log |\cT|$ but $k \simeq \log |\cT|$, then we gain orders of magnitude improvements through this upper bound in the adaptive data analysis. Moreover, for any fixed $k$, the lower bound also holds if $\cT$ sufficiently rich. This richness can often be measured in terms of dimension or $\log$-cardinality.

To be more concrete on the above discussion, we consider the following simple linear example that satisfies these assumptions, but clearly does not contain all functions that satisfy them.
\begin{example}\label{exp:linear}
	Let $X$ be a $\R^{d\times n}$, each column $X_i$ is drawn iid from a multivariate normal distribution.	
	Moreover, assume the features of $X$ are appropriately normalized so that the variances are equal to $\sigma^2$. Here 
	$$\cT = \left\{ t \in \R^d: \|t\|_2\leq 1 \right\}$$ is the class of all unit vectors and $\phi_{t}(X) = \frac{1}{n} \sum_{i=1}^{n}\langle t,X_i\rangle$. It is clear that for any $t\in \cT$, $\Var(\phi_t(X)) \leq \frac{\sigma^2}{n}$. Also, for any fixed finite subset of $\{t_1,...,t_k\}\subset \cT$, $\phi_{t_1,...,t_k}(X)$ is a multivariate normal distribution.	
\end{example}
Suppose each feature of $X$ is independent, when $k < d^2$ the upper bound of the risk $\frac{\sqrt{k}\sigma^2}{n}$ using adaptive data analysis is meaningful. Our lower bound uses a construction that requires $k-1$ independent queries, but there are at most $d$ independent queries in this example. As a result, we can only get a matching lower bound for $k<d$. 
It is unclear whether the upper bound can be attained in the region when $d <k< d^2$.

On the other hand, 	when $k\gg d^2$, we know for sure that the upper bound is no longer tight since by standard uniform convergence, we can get 
\begin{equation}\label{eq:UGC_upper_bound}
  \E \max_{t\in \cT} |\phi_{t}(X)|^2 = O\left(\frac{d\sigma^2}{n}\right).
  \end{equation}
When $X$ has additional properties such as sparsity, fast decaying spectrum, \textit{etc}, we can replace $d$ with much smaller quantities that captures the essential degree of freedom of the problem.
Such a bound corresponds to revealing the entire data set to the adversary so that he can try out every single function in $\cT$ or equivalently when $k\rightarrow \infty$. The lower bound for the case when $k>d$ remains an open problem.
We conjecture that the minimax rate grows at a slower rate than $\sqrt{k}$ when $k>d$ and attains the upper bound given in \eqref{eq:UGC_upper_bound} when $k = \Omega(2^d)$. 

\section{Conclusion}

In this paper, we presented a minimax framework for adaptive data analysis and derived the first minimax lower bound that matches the upper bound in \citet{russo2015controlling} up to constant.  We also presented an elementary proof for the same upper bound for the Gaussian noise-adding procedure. Our results reveal that the minimax risks for $1$-step and $k$-step adaptive data analyses are on the same order, and an approximate least favorable $1$-step adversary maximizes the bias by choosing a query that is simultaneously correlated with all previous queries with the signs of the covariance vector decided by an optimal classifier.
In the discussion, we pointed out the implicit dependency of the lower bound on the richness of the query class. Through an illustrative example, we discussed the intriguing regime that interpolates the risk bounds via adaptive data analysis and uniform convergence.

Future work includes finding the upper bound for the subgaussian case, strengthening the lower bound to allow an even large class of estimators as well as extending the minimax framework for more practical regimes.


\bibliographystyle{apa-good}
\bibliography{personal-privacy}

\appendix

\section{Technical proofs}
\subsection{Proof of Theorem~\ref{thm:one-step_adaptive}}\label{sec:proof_one_step}
\begin{proof}
	In this proof, we denote short hand $\phi:=\phi_{t_{1:k-1}}$, $\mu:=\mu_{t_{1:k-1}}$ and $Z:=Z_{1:k-1}$.
	We first bound the bias term using law of total expectations
	$$
	\E_{T_k,\phi_{T_k}} (\phi_{T_k}-\mu_{T_k}) = \E_{\phi,Z}\E_{T_k\sim \cW(\phi+Z)} \E_{\phi_{T_k}}(\phi_{T_k}-\mu_{T_k}\big| T_k,\phi+Z )\leq \E_{\phi,Z}\sup_{t_k\in \cT} \E_{\phi_{t_k}}(\phi_{t_k}-\mu_{t_k}\big| \phi+Z )
	$$
	By Gaussianity, the conditional expectation is 
	$$
	\E_{\phi_{t_k}}(\phi_{t_k}-\mu_{t_k}\big| \phi + Z) = \Sigma_{k,1:k-1}^T (\Sigma_{1:k-1} + w^2I)^{-1} (\phi+Z -\mu)
	$$

	What is the worst $\Sigma_{1:k-1,k}$ to use? If unconstrained the above bias can go to infinity.
	The choice of $\Sigma_{1:k-1,k}$ must obey that after we augment $\Sigma_{1:k-1}$ with it $\Sigma$ remains positive semidefinite.
	In general, suppose we have block matrix $M =\begin{bmatrix}
	\lambda & v^T\\
	v & \Sigma
	\end{bmatrix}$
	where $\Sigma$ is positive definite, then 
	$$M\succ 0 \Leftrightarrow \det(M) = \det(A)\det(A-v \lambda^{-1} v^T)> 0.$$
	which tells us that the optimal solution to 
	$$
	\left\{\max_{v,\lambda} \langle v, x\rangle \Big| M\succ 0, \lambda \leq \sigma^2  \right\}$$
	is to take $\lambda = \sigma^2$ and the problem translates into 
	$$ \left\{\max_{v} \langle v, x\rangle \Big| vv^T \prec \sigma^2 \Sigma \right\} = \left\{\max_{v} \langle v, x\rangle \Big| \sqrt{v^T \Sigma^{-1}v} < \sigma\right\} = \|v^*\|_{\Sigma^{-1}}\|x\|_\Sigma = \sigma \|x\|_\Sigma$$
	where $\|\cdot\|_\Sigma$ is the dual norm of $\|\cdot\|_{\Sigma^{-1}}$. 
	
	Take $\Sigma=\Sigma_{1:k-1}$ and we use the above argument can work out $\Sigma_{k,1:k-1}$ that attains the supremum and the corresponding conditional bias is
	$$
	\sigma\|(\Sigma_{1:k-1} + w^2I)^{-1} (\phi+Z -\mu)\|_{\Sigma_{1:k-1}}.
	$$
	Since $(\cdot)^2$ is monotonically increasing on $\R_+$, the supremum of the square conditional bias is attained by the same choice of $t_k$ and 
	\begin{align}
	&\E_{\phi,Z}\sigma^2 (\phi +Z-\mu)^T(\Sigma_{1:k-1} + w^2I)^{-1} \Sigma_{1:k-1}(\Sigma_{1:k-1} + w^2I)^{-1} (\phi+Z -\mu)\nonumber\\
	=&\sigma^2 \tr\left\{ \E_{\phi,Z} \left[(\phi +Z-\mu) (\phi+Z -\mu)^T\right](\Sigma_{1:k-1} + w^2I)^{-1} \Sigma_{1:k-1}(\Sigma_{1:k-1} + w^2I)^{-1}\right\}\nonumber\\
	=&\sigma^2 \tr\left\{ \Sigma_{1:k-1}(\Sigma_{1:k-1} +w^2I)^{-1} \right\}\nonumber\\
	=&\sigma^2 \sum_{i=1}^{k-1} \frac{\lambda_i}{\lambda_i+w^2} \leq \frac{\sigma^2}{w^2} \sum_{i=1}^{k-1}\lambda_i = \frac{(k-1)\sigma^4}{w^2}. \label{eq:1step_derivation1}
	\end{align}	
	where $\lambda_i$ are the eigenvalues of $\Sigma_{1:k-1}$. The last line diagonalizes $\Sigma_{1:k-1}$ and uses the fact that trace operator is unitary invariant. Note that the inequality is sharp up to a small constant as for the case when $\lambda_i=\sigma^2$ for all $i$, the quantity is equal to $\frac{(k-1)\sigma^2}{\sigma^2 + w^2}$. By Jensen's inequality, the upper bound for the expected conditional bias in \eqref{eq:1step_derivation1} is also upper bounds the bias.
	
	The proof of the second claim is simply decomposing the square error of $A_k = \phi_{T_k}+Z_k$ into square bias and variance, and upper bound each term by choosing  $w=\sqrt{k-1}\sigma^2$.
	
	The bias of $A_k$ is the same as that of $\phi_{T_k}$ in \eqref{eq:1step_derivation1}. The variance obeys
	\begin{align*}
	&\Sigma_k +w^2 - \E_{\phi,Z}\sup_{\phi_k}\Sigma_{k,1:k-1}^T (\Sigma_{1:k-1} + w^2I_{1:k-1})^{-1} \Sigma_{k,1:k-1}\\
	\leq & \Sigma_k +w^2 \leq \sigma^2 + w^2 \leq (\sqrt{k-1}+1)\sigma^2
	\end{align*}
	Adding the two upper bounds give us the right form.	
\end{proof}

\subsection{Proof of Theorem~\ref{thm:k-step_adaptive}}\label{sec:proof_k_step_upper}
\begin{proof}
	We first control the bias. As we worked out in the proof of Theorem~\ref{thm:one-step_adaptive},
	\begin{align}
	\E \left(\phi_{T_k}  - \mu_{T_k} \right) =& \E_{T_{1:k-1}, A_{1:k-1}} \E_{T_k\sim \cW_k|T_{1:k-1},A_{1:k-1}}\left(\E_{\phi_{T_k} | T_{1:k},A_{1:k-1}}\phi_{T_k}  - \mu_{T_k} \right)\nonumber\\
	\leq &\E_{T_{1:k-1}, A_{1:k-1}} \sup_{t_k\in \cT} \left(\E_{\phi_{t_k} | T_{1:k-1},t_k,A_{1:k-1}}\phi_{t_k}  - \mu_{T_k} \right)\nonumber\\
	= &\E_{T_{1:k-1},A_{1:k-1}} \sqrt{\sigma^2 \mb f_{k-1} }
	\leq  \sqrt{\sigma^2 \E_{T_{1:k-1},A_{1:k-1}} \mb f_{k-1}},\label{eq:k-step_derive1}
	\end{align}
	where for simplicity, we denote 
	$$\mb f_{k-1}:= (A_{1:k-1} -\mu_{T_{1:k-1}})^T (\Sigma_{1:k-1}+W_{1:k-1})^{-1}\Sigma_{1:k-1} (\Sigma_{1:k-1}+W_{1:k-1})^{-1}(A_{1:k-1} -\mu_{T_{1:k-1}}),$$
	and the last step follows from the Jensen's inequality on the concave function $\sqrt{\cdot}$.
	Note that every variable in $\mb f_{k-1}$ is a random variable.
	
	We will further expand the above expectation into a sequence of expectations and recursively evaluate the conditional expectation and then taking supremum  of $\mb f_{k-1}$.
	\begin{align*}
	\E \mb f_{k-1} &= \E_{T_{1:k-2},A_{1:k-2}}\E_{T_{k-1}|T_{1:k-2},A_{1:k-2}}\E_{A_{k-1}|T_{1:k-1},A_{1:k-2}}\mb f_{k-1}\\
	&\leq \E_{T_{1:k-2},A_{1:k-2}}\sup_{\phi_{t_{k-1}}\in \cT}\E_{A_{k-1}|T_{1:k-2},t_{k-1},A_{1:k-2}}\mb f_{k-1}
	\end{align*}
	It turns out that we can neatly express the conditional expectation in a closed form as a function of $\Sigma_{1:k-1}$ and the diagonal covariance of the added noise $W_{1:k-1}$.
	
	\begin{lemma}\label{lem:recursive_fk}
		Denote $w^2 = W_{k-1}, W = W_{1:k-2}$ and $\Sigma = \Sigma_{1:k-2}, v = \Sigma_{k-1,1:k-2}, \lambda = \Sigma_{k-1}$  such that 
		\begin{align*}W_{1:k-1} =  \begin{bmatrix}
		W      & 0 \\
		0 & w^2
		\end{bmatrix}, &&
		\Sigma_{1:k-1} =  \begin{bmatrix}
		\Sigma      & v \\
		v^T & \lambda
		\end{bmatrix}.
		\end{align*} In addition, $\Omega := (\Sigma+W)^{-1}\Sigma(\Sigma+W)^{-1}$. We have
		\begin{equation}\label{eq:recursive_f_k}
		\E_{A_{k-1}|T_{1:k-2},t_{k-1},A_{1:k-2}}\mb f_{k-1} = \mb f_{k-2} + \frac{(\lambda + v^T\Omega v - v^T(\Sigma+W)^{-1}v)(\lambda + w^2)}{ (\lambda + w^2 - v^T(\Sigma + W)^{-1}v)^2}.
		\end{equation}
	\end{lemma}
	In order to not interrupt the flow of the arguments, we defer the proof of Lemma~\ref{lem:recursive_fk} to the appendix.
	
	With this parametric form, the supremum can be rewritten as
	\begin{equation}\label{eq:sup_fk_eval}
	\begin{aligned}
	&\sup_{T_{k-1}\in \cT}\E_{A_{k-1}|T_{1:k-2},t_{k-1},A_{1:k-2}}\mb f_{k-1} \\
	=& \mb f_{k-2} + \max_{\lambda \leq \sigma^2, \Sigma_{1:k-1}\succeq 0}\frac{(\lambda + v^T\Omega v- v^T(\Sigma+W)^{-1}v)(\lambda + w^2)}{ \left(\lambda + w^2 - v^T(\Sigma + W)^{-1}v\right)^2}.
	\end{aligned}
	\end{equation}

	As in the previous calculations, the semidefinite constraint requires that $v^T\Sigma^{-1}v \leq \lambda$, also $(\Sigma+W)^{-1} \preceq \Sigma^{-1}$, therefore for any $v$, 
	$$v^T\Omega v = v^T(\Sigma+W)^{-1}\Sigma(\Sigma+W)^{-1}v \leq  v^T(\Sigma+W)^{-1}v \leq v^T\Sigma^{-1}v\leq \lambda.$$
	Substitute into \eqref{eq:sup_fk_eval}, we get an upper bound of the supremum 
	$$
	\sup_{T_{k-1}\in \cT}\E_{A_{k-1}|T_{1:k-2},t_{k-1},A_{1:k-2}}\mb f_{k-1} \leq \mb f_{k-2} +\max_{\lambda\leq \sigma^2} \frac{\lambda(\lambda+w^2)}{w^4} \leq \mb f_{k-2}+ \frac{\sigma^2}{w^2} + \frac{\sigma^4}{w^4}.
	$$
	Recursively evaluating and upper bounding the supremum until the base case
	$$\sup_{t_1 \in \cT}\E_{A_1|t_1}\mb f_{1} = \max_{\sigma_1^2\leq \sigma^2}\frac{\E(A_1- \mu_{t_1})^2\sigma_1^2}{(\sigma_1^2 + w_1^2)} = \max_{\sigma_1^2\leq \sigma^2}\frac{\sigma_1^2}{\sigma_1^2+w_1^2} \leq \frac{\sigma^2}{w_1^2} + \frac{\sigma^4}{w_1^4},$$
	we end up with
	$$
	\E \mb f_{k-1} =  \sigma^2\sum_{i=1}^{k-1}\left(\frac{1}{w_i^2} + \frac{\sigma^2}{w_i^4}\right),
	$$
	and substitute into \eqref{eq:k-step_derive1}, we obtain an upper bound for the bias
	$$
	\E \left(\phi_{T_k}  - \mu_{T_k} \right) \leq \sqrt{\sigma^4\sum_{i=1}^{k-1}\frac{1}{w_i^2} + \frac{\sigma^2}{w_i^4}}.
	$$
	This gives the desired bound for the first claim.
	
	The variance can be easily evaluated using the conditional variance.
	\begin{align*}
	\Var(\phi_k)  &= \E_{T_{1:k-1},A_{1:k-1}} \E_{T_k|T_{1:k-1},A_{1:k-1}}\E_{\phi_{T_k} |T_{1:k},A_{1:k-1}}\left(\phi_{T_k}  - \E_{\phi_{T_k}} \right)^2\\ 
	&= \E_{T_{1:k-1},A_{1:k-1}} \E_{T_k|T_{1:k-1},A_{1:k-1}} (\sigma_i^2 - \Sigma_{k,1:k-1}\Sigma_{1:k-1}\Sigma_{1:k-1,k}) \leq \sigma^2.
	\end{align*}
	Combining the bounds for bias and variance, we upper bound the mean square error by 
	$$
	\sigma^4\sum_{i=1}^{k-1}\left(\frac{1}{w_i^2} + \frac{\sigma^2}{w_i^4}\right)+\sigma^2.
	$$
The second claim follows directly by taking $w_i^2 = \sqrt{k-1}\sigma^2$ for each $i=1,..,k-1$, and add the variance of the additional noise $Z_k$. 
\end{proof}
\subsection{Proof of Lemma~\ref{lem:recursive_fk}}
\begin{proof}
	We prove by a direct calculation. First of all, we invert $(\Sigma_{1:k-1}+W_{1:k-1})$ in block form and use Sherman-Morrison formula on the first principle minor:
	\begin{align*}
	&(\Sigma_{1:k-1}+W_{1:k-1})^{-1} = \begin{bmatrix}
	\Sigma +W     & v \\
	v^T & \lambda + w^2
	\end{bmatrix}^{-1}\\
	=&\begin{bmatrix}
	\left(\Sigma + W - \frac{vv^T}{\lambda + w}\right)^{-1}  & -(\Sigma+W)^{-1}v\left(\lambda + w^2 - v^T(\Sigma+W)^{-1}v\right)^{-1} \\
	-\left(\lambda + w^2 - v^T(\Sigma+W)^{-1}v\right)^{-1}v^T(\Sigma+W)^{-1} & \left(\lambda + w^2 - v^T(\Sigma+W)^{-1}v\right)^{-1}
	\end{bmatrix}\\
	=&\begin{bmatrix}
	(\Sigma + W)^{-1} + \alpha(\Sigma + W)^{-1}vv^T(\Sigma + W)^{-1}  & -\alpha (\Sigma+W)^{-1}v \\
	-\alpha v^T(\Sigma+W)^{-1} & \alpha
	\end{bmatrix}
	\end{align*}
	where we denote $\alpha:=\left(\lambda + w^2 - v^T(\Sigma+W)^{-1}v\right)^{-1}$.
	
	For any symmetric block matrices
	\begin{align*}
	&\begin{bmatrix}
	X & Y \\
	Y^T & Z
	\end{bmatrix}
	\begin{bmatrix}
	A & B \\
	B^T & C
	\end{bmatrix}
	\begin{bmatrix}
	X & Y \\
	Y^T & Z
	\end{bmatrix}\\
	=& 
	\begin{bmatrix}
	XAX + XBY^T + Y^TBX + YCY^T & XAY + XBZ + YB^TY + YCZ \\
	Y^TAX + Y^TBY^T + ZB^TX + ZCY^T & Y^TAY + Y^TBZ + ZB^TY + ZCZ
	\end{bmatrix}
	\end{align*}

	For the special case here let $A_{1:k-2}-\mu_{T_{1:k-2}} =: x$ and $A_{k-1}-\mu_{T_{k-1}} =:y$,
	$$\mb f_{k-1}  = \begin{bmatrix}
	x^T&y^T
	\end{bmatrix}\begin{bmatrix}
	F_{11}&F_{12}\\
	F_{21}&F_{22}
	\end{bmatrix}\begin{bmatrix}
	x\\y
	\end{bmatrix} = x^T F_{11} x + 2 x^T F_{12} y  + y^TF_{22}y$$
	for some blocks $F_{\cdot,\cdot}$. 
	
	Further denote  $b:=v^T(\Sigma+W)^{-1}x$, $\Omega:=(\Sigma+W)^{-1}\Sigma(\Sigma+W)^{-1}$,
	$$
	\begin{cases}
	x^TF_{11}x =& x^T\Omega x + 2\alpha b (x^T\Omega v) + \alpha^2b^2 (v^T\Omega v)  - 2\alpha b^2 - 2\alpha^2b^2 [v^T(\Sigma+W)^{-1}v] +\lambda \alpha^2b^2.\\
	2x^TF_{12}y=&-2\alpha b(x^T\Omega v) - 2\alpha^2 b^2 (v^T\Omega v) + 2\alpha b^2 + 4\alpha^2 b^2 [v^T(\Sigma+W)^{-1}v]  - 2\lambda \alpha^2b^2\\
	y^TF_{22}y =& \lambda \alpha^2b^2 + \alpha^2b^2 (v^T\Omega v)  - 2\alpha^2b^2[v^T(\Sigma + W)^{-1}v] \\
	& + \alpha^2 \left[ \lambda + v^T\Omega v - v^T(\Sigma+W)^{-1}v\right] (\lambda+w^2)
	\end{cases}
	$$
	It's easy to check that $x^T\Omega x = \mb f_{k-2}$ and almost everything cancels out when we sum the three terms up. All that remains gives exactly what we claim to be true.
\end{proof}
\subsection{Proof of Theorem~\ref{thm:lower_bound}}
\begin{proof}
	The proof constructs one specific sequence of $\cW_{1:k}$ and show that, the best sequence of randomization $\cA_{1:k}$ cannot bring the risk down to anything smaller than $\Omega(\sqrt{k}\sigma^2)$.
	
	Specifically, let $\cW_{1:k-1}$ pick $T_{1:k-1}$ such that $\phi_{T_{1:k-1}}(X) \sim \cN(\mu_{T_{1:k-1}},\sigma^2I)$, namely, $\phi_{T_{1:k}}$ are independent.
	And we take $\cW_k$ to be the following algorithm:
	\begin{enumerate}
		\item 	Based on the known distributions of $Z_1, Z_2|\bar{H}_2,...,Z_{k-1}|\bar{H}_{k-1}$ declared by the player, conduct the likelihood ratio test that output $\hat{s}_i = 1$ if
		$$
		\log \frac{\P(A_i|\phi_{T_i}-\mu_{T_i} > 0)}{\P(A_i|\phi_{T_i}-\mu_{T_i} < 0)} \geq 0
		$$
		and $\hat{s}_i = -1$ otherwise. Since $A_i = \phi_{T_i}+Z_i$ and $\mu_{T_i}$ is known to the adversary, this is equivalent to observing $\phi_{T_i}-\mu_{T_i}+Z_i$. Let the estimated sign vector be $\hat{s}$. 
		\item Furthermore, based on the known function that produces a distribution of $Z_k$ given any $\phi_{T_k}$, choose $T_k=t_k$ such that $\phi_{t_k}$ has mean $0$, variance $\sigma^2$ and covariance vector with known realized $T_{1:k-1}$ being $v=\frac{\sigma^2}{\sqrt{k-1}}\hat{s}$ if $\E_{Z_k|\bar{H}_k}(Z_k) >0$ and $-v$ otherwise.
	\end{enumerate}
	
	Denote $\E(Z_i|\bar{H}_i) =: b_i$ and $\E[(Z_i-b_i)^2|\bar{H}_i] =: w_i^2$. The first step of our proof is a reduction that says we can restricts our attention to only zero mean noises for $i=1,...,k-1$.
	
	Observe that since $b_i$ is revealed to the adversary, adding noise with $b_i\neq 0$ does not change $\cW_{k}$ at all. Therefore it suffices to show that choosing $b_i\neq 0$ only increases $\E (A_{T_i}-\mu_{T_i})^2$ for $i=1,...,k-1$. Conditioned on the entire history $\bar{H}_i$, $\phi_{T_i}$ and $Z_i$ are independent. Therefore, 
	$$
	\E (A_{T_i}-\mu_{T_i})^2=\E_{\bar{H}_i}\E_{T_i,\phi_{T_i},Z_i|\bar{H}_i}(\phi_{T_i}-\mu_{T_i} + Z_i)^2=\sigma^2 +\E_{\bar{H}_i} (b_i^2 + w_i^2) \geq \sigma^2 + \E_{\bar{H}_i} (w_i^2).
	$$
	where the equal sign is attained if $\E(Z_i|\bar{H}_i) = 0$ for any $\bar{H}_i$.
	
	Now we move on to deal with the last term.
	\begin{align*}
	\E(A_k - \mu_{T_k})^2 &= \E_{\bar{H}_k}\E_{Z_k|\bar{H}_k}\E_{T_k,\phi_{T_k}|\bar{H_k}} (\phi_{T_k} - \mu_{T_k} + Z_k)^2\\
	&=  \E_{\bar{H}_k} [ \E_{T_k,\phi_{T_k}|\bar{H}_k}(\phi_{T_k} - \mu_{T_k}) + b_k]^2 + \E_{\bar{H}_k} \Var(\phi_{T_k} + Z_k|\bar{H}_k)\\
	&=\E_{\bar{H}_k} [ \E_{T_k,\phi_{T_k}|\bar{H}_k}(\phi_{T_k} - \mu_{T_k}) + b_k]^2 + \E_{\bar{H}_k}\Var(\phi_{T_k}|\bar{H}_k) + \E_{\bar{H}_k} w_k^2\\
	&\geq 	\E_{\bar{H}_k} [ \E_{T_k,\phi_{T_k}|\bar{H}_k}(\phi_{T_k} - \mu_{T_k})]^2 + \E_{\bar{H}_k}\Var(\phi_{T_k}|\bar{H}_k)		 
	\end{align*}
	The third line uses the property that $\phi_{T_k} \independent (Z_k-b_k) | \bar{H}_k$, and the fourth line drops the variance and uses our choices of $\cW_k$ that ensures $\E_{T_k,\phi_{T_k}|\bar{H}_k}(\phi_{T_k} - \mu_{T_k})$ and $b_k$ to always have the same sign. The equal sign in the last inequality is attained by taking $Z_k \equiv 0$. We further drop the variance term and use Jensen's inequality to get
	$$
	\E(A_k - \mu_{T_k})^2 \geq [ \E(\phi_{T_k} - \mu_{T_k})]^2.
	$$
	With that, we can lower bound the minimax risk using
	\begin{equation}\label{eq:F1k_intermediate}
	F_k \geq \min_{\cZ_{1:k-1} \text{ zero-mean.}} \left\{ \max_{i\in[k-1]} (\sigma^2 + \E_{\bar{H}_i} (w_i^2)) \vee (\E\phi_{T_k} - \mu_{T_k})^2 \right\}.
	\end{equation}
	
	This lower bound suggests that we can focus on finding an lower bound of the bias induced by the adaptivity in $\cW_k$.
	
	Let $t_1,...,t_{k-1}$ be any queries chosen by $\cW_{1:k-1}$, as they are non-adaptive, we can choose them ahead of time. Let $T_k$ be chosen according to $\cW_k$. This is a deterministic map so the distribution of $T_k$ is completely induced by the randomness in $\phi_{t_{1:k-1}}$ (randomness in data $X$) and the randomness in $Z_{1:k-1}$. We denote the sign predictor from $\cW_k$ by $\hat{s}\in \{-1,1\}^{k-1}$. Note that $\cW_k$ makes the prediction based on the actual noise added 
	
	\begin{align}
	\left|\E \phi_{T_k} - \mu_{T_k}\right|&=  \left|\E_{\phi_{t_{1:k-1}},Z_{1:k-1}} \E_{\phi_{T_k}\big| T_k =\cW_k(A_{t_{1:k-1}}),\phi_{t_{1:k-1}}}(\phi_{T_k} - \mu_{T_k}) \right|\nonumber\\
	&=\left|\E_{\phi_{t_{1:k-1}},Z_{1:k-1}} \left( \frac{\sigma^2}{\sqrt{k-1}}\hat{s}^T\sigma^{-2}(\phi_{t_{1:k-1}}-\mu_{t_{1:k-1}})\right) \right|\nonumber\\
	&=\left|\E_{\phi_{t_{1:k-1}},Z_{1:k-1}}\frac{\hat{s}^T(\phi_{t_{1:k-1}}-\mu_{t_{1:k-1}})}{\sqrt{k-1}}\right|\nonumber \\
	&=\frac{1}{\sqrt{k-1}}\left| \sum_{i=1}^{k-1} \E_{\bar{H}_i}\E_{Z_i|\bar{H}_i} \E_{\phi_{t_{i}}}\hat{s}_i(\phi_{t_{i}}-\mu_{t_{i}})  \right|\nonumber\\
	&= \frac{1}{\sqrt{k-1}}\left| \sum_{i=1}^{k-1} \E_{\bar{H}_i}\max_{\hat{s}_i \in \cC}\E_{Z_i|\bar{H}_i}\E_{\phi_{t_{i}}} \hat{s}_i(\phi_{t_{i}}-\mu_{t_{i}})  \right| \nonumber\\
	&\geq \frac{1}{\sqrt{k-1}}\left| \sum_{i=1}^{k-1} \max_{\tilde{s}_i \in \cC} \E_{\bar{H}_i}\E_{Z_i|\bar{H}_i}\E_{\phi_{t_{i}}} \tilde{s}_i(\phi_{t_{i}}-\mu_{t_{i}})  \right| \label{eq:convert_to_classifier}
	\end{align}			
	Where $\cC$ is the class of all functions $\R \rightarrow \{-1,1\}$. The fifth line is due to our choice of $\cW_k$ that uses an optimal classifier and the last line follows from Jensen's inequality and the convexity of pointwise maximum.
	
	Since $\phi_{t_i}$ is independent to $\bar{H_i}$, it is equivalent to releasing $A_i = \phi_{t_i} + Z_i$ where $Z_i$ is drawn from the marginal distribution with $\bar{H}$ integrated out. 
	$$\Var{Z_i} = \E_{\bar{H}_i} (\Var{Z_i|\bar{H}_i}) + \Var(\E(Z_i|\bar{H}_i)) = \E_{\bar{H}_i} (\Var{Z_i|\bar{H}_i}) = \E w_i^2.$$
	
	
	Here we will invoke the following lemma on the optimal obfuscation of a Bayes classifier (which we defer the proof to later). 
	\begin{lemma}\label{lem:best_bayes_error_normal}
		Let signal random variable $X$ and noise random variable $Z$ be such that  $\E Z=0$, $\Var(Z) \leq w^2$, $x\sim \cN(0,\sigma^2)$. Let $\hat{s}$ be the optimal Bayes classifier of $\sign(X)$ by observing $X+Z$, then 
		\begin{equation}\label{eq:best_bayes_error_normal}
		\E_{X,Z}  (\hat{s} X) \geq \begin{cases}
		\frac{\sigma^2}{\sqrt{3}w}- \frac{\sigma^4}{2\sqrt{3}w^3}&	\text{ when }w^2 \geq \sigma^2,\\
		\frac{\sigma}{2\sqrt{3}}&	\text{ when }w^2 < \sigma^2.
		\end{cases}
		\end{equation}
		Moreover, 
		$$
		\lim_{\frac{\sigma^2}{w}\rightarrow 0}  \frac{w}{\sigma^2}\E_{X,Z}  (\hat{s} X) \geq \frac{1}{\sqrt{3}}
		$$
		with equal sign attained by uniform distribution $U([-\sqrt{3}w,\sqrt{3}w])$.
	\end{lemma}
	
	For $w>\sigma$, we can relax the lower bound \eqref{eq:best_bayes_error_normal} further into $\frac{\sigma^2}{2\sqrt{3}w}$. Take $w=\sqrt{\E w_i^2}$ for each $i$ and substitute into \eqref{eq:convert_to_classifier}, we obtain
	$$
	\left|\E \phi_{T_k} - \mu_{T_k}\right| \geq \frac{1}{\sqrt{k-1}} \sum_i C\min\left\{\frac{\sigma^2}{\sqrt{\E w_i^2}}, \sigma\right\}.
	$$
	for a universal constant $C= \frac{1}{2\sqrt{3}}$. The square bias obeys
	$$\left|\E \phi_{T_k} - \mu_{T_k}\right|^2 \geq \frac{C^2}{k-1} \left[\sum_i \min\left\{\frac{\sigma^2}{\sqrt{\E w_i^2}}, \sigma\right\}\right]^2 \geq \frac{C^2}{k-1} \sum_i\min\left\{ \frac{\sigma^4}{\E w_i^2},\sigma^2\right\}$$
	
	Substitute the lower bound into \eqref{eq:F1k_intermediate}
	$$
	F_k \geq \min_{\E w_1^2,...,\E w_{k-1}^2} \left\{ \max_{i\in[k-1]} \E w_i^2 \vee \frac{1}{k-1} \left[\sum_{i=1}^{k-1} C\min\left\{\frac{\sigma^2}{\sqrt{\E w_i^2}}, \sigma\right\} \right]^2 \right\}.
	$$
	The first $k-1$ term is monotonically increasing in $\E w_i^2$, the second term is monotonically decreasing in $\E w_i^2$ for each $i$. The minimizer occurs when the $k$ terms are all equal, which appears when $\E w_i^2 = C\sqrt{k-1}\sigma^2$. This completes the proof.
\end{proof}
\subsection{Proof of Lemma~\ref{lem:best_bayes_error_normal}}
\begin{proof}
	$\E \hat{s} X$, the absolute margin risk of a classifier $\hat{s}$, is a function of the noise distribution $p$. For example, if $p=0$, $\hat{s}\equiv\sign(X)$ therefore $\E \hat{s} X = \E |X|$, if $p$ is normal with variance $w^2\rightarrow \infty$, $\hat{s}$ is independent to the signs of $X$, therefore this quantity converges to $0$. The specific shape of $p$ also matters, e.g., adding Bernoulli noise with $w^2 \rightarrow \infty$ yields $\E \hat{s} X = \E |X|$. The idea of the proof is to formulate an optimization problem that minimizes $\E \hat{s} X$ over the class of all $p$ and then try to solve it analytically.
	
	To begin with, we first express $\E \hat{s} X$ as the $L_1$ norm of a linear transformation of $p$. Decompose $X$ into sign $s$ and magnitude $t$, where $\P(s=1)=\P(s=-1)=0.5$ and $t$ is drawn from a half-normal distribution which we denote by $q$.
	\begin{align}&\E_{Z}\E_{X}\hat{s} X\nonumber
	=\E_{Z}\E_{s} \E_{t} t  s\hat{s}\nonumber\\
	=&\E_{Z}\sum_{s\in\{-1,1\}}0.5 \E_{t} t  s \;\sign[\E_{t'}\P(A|s'>0,t') - \E_{t'} \P(A|s'<0,t')>0]\nonumber\\
	=& 0.5 \E_{Z} \E_{t}t 1_{\{t + Z \in E_1\}}  -  0.5 \E_{Z} \E_{t}t 1_{\{t + Z \in E_2\}}  -0.5 \E_{Z} \E_{t}t 1_{\{-t + Z \in E_1\}} + 0.5 \E_{Z} \E_{t}t 1_{\{-t + Z \in E_2\}}\nonumber\\
	=& 0.5 \E_{t}\E_{A|t,s=1} t 1_{\{A \in E_1\}} - 0.5 \E_{t}\E_{A|t,s=1} t 1_{\{A_1 \in E_2\}}   - 0.5 \E_{t}\E_{A|t,s=-1} t 1_{\{A \in E_1\}} + 0.5 \E_{t}\E_{A|t,s=-1} t 1_{\{A \in E_2\}}\nonumber\\
	=&0.5 \int_{z}  \left[\int_{t} t (p(z-t) - p(z+t)) q(t) dt\right]_+ dz + 0.5 \int_{z}\left[\int_{t} t (-p(z-t) + p(z+t)) q(t) dt \right]_+ dz\nonumber\\
	= & 0.5 \int_{z}  \left| \int_t t (p(z-t) - p(z+t))q(t)  dt \right|  dz = 0.5 \int_{z}  \left| \E_t t (p(z-t) - p(z+t)) \right|  dz \label{eq:Pe_as_1norm}
	\end{align}
	where in Line 3 and 4, we use $E_1$ to denote the event of $A$ such that $\sign[\E_{t'}\P(A|s'>0,t') - \E_{t'} \P(A|s'<0,t')>0] = 1$ and $E_2 = E_1^c$.  Note that $E_1$ and $E_2$ are events in the $\sigma$-field of observation $A$ induced only by the $\sigma$-field of $Z$ (since $X$ is integrated out).

	Consider the following variational optimization problem over distribution $p$ that is $0$-mean and has variance bounded by $w^2$.
	\begin{equation}
	\begin{aligned}
	\min_{p } &  	 
	\int |\E_t t p(x+t)  - \E_t t p(x-t)|dx\\
	\text{s.t. } & p \text{ is a probability distribution defined on }\R, \\
	&\Var(Z)\leq w^2, \E(Z) = 0 \text{ for } Z\sim p.
	\end{aligned}
	\end{equation}
	where $t$ distributes as half-normal distribution with parameter $\sigma$.
	
	Define operator $A$ such that $Ap=  \int_t t[p(x+t)-p(x-t)] \frac{\sqrt{2}}{\sigma\sqrt{\pi}}e^{-\frac{t}{2\sigma^2}}dt $. The objective can be rewritten as $\|A p\|_1$. $A$ is a linear operator, all constraints are affine in $p$, therefore this is a convex optimization problem, which we rewrite in standard form below:
	\begin{equation}\label{eq:opt_over_p}
	\begin{aligned}
	\min_{\mb p } &  	 
	\|A \mb p\|_1\\
	\text{s.t. } & \langle \mb x^2, \mb p\rangle\leq w^2, \quad  -\mb p\leq 0\\
	& \langle \mb 1, \mb p\rangle =1,\quad
	\langle \mb x,  \mb p\rangle = 0.
	\end{aligned}
	\end{equation}
	The Lagrangian and the corresponding dual problem are
	$$
	L(\mb p,u_1,\mb u_2,v_1,v_2) = \|A \mb p\|_1 + u_1 (\langle \mb x^2, \mb p\rangle - w^2) - \langle \mb u_2, \mb p\rangle  + v_1 (1- \langle \mb 1, \mb p\rangle )   + v_2 \langle \mb x,  \mb p\rangle,
	$$
	\begin{equation}\label{eq:dual_opt_over_p}
	\begin{aligned}
	\max_{u1,\mb u_2,v_1,v_2, C } & -u_1 w^2 + v_1 \\
	\text{s.t. } & \| A^{-1}(-u_1\mb x^2 + \mb u_2 + v_1\mb 1 -v_2 \mb x ) + C\|_\infty\leq 1\\
	&  \mb u_2 \geq \mb 0,\quad  u_1\geq 0.
	\end{aligned}
	\end{equation}
	and by the definition of the Lagrange dual, the corresponding dual objective value for any feasible dual variables will be a lower bound of the primal optimal solution, and our proof involves constructing one ``nearly optimal'' feasible dual solution. In the derivations below, please refer to Figure~\ref{fig:dualfunctions} for illustrations.
	\begin{figure}[htbp]
		\centering
		\subfigure[][t]{
			\includegraphics[width=0.45\textwidth]{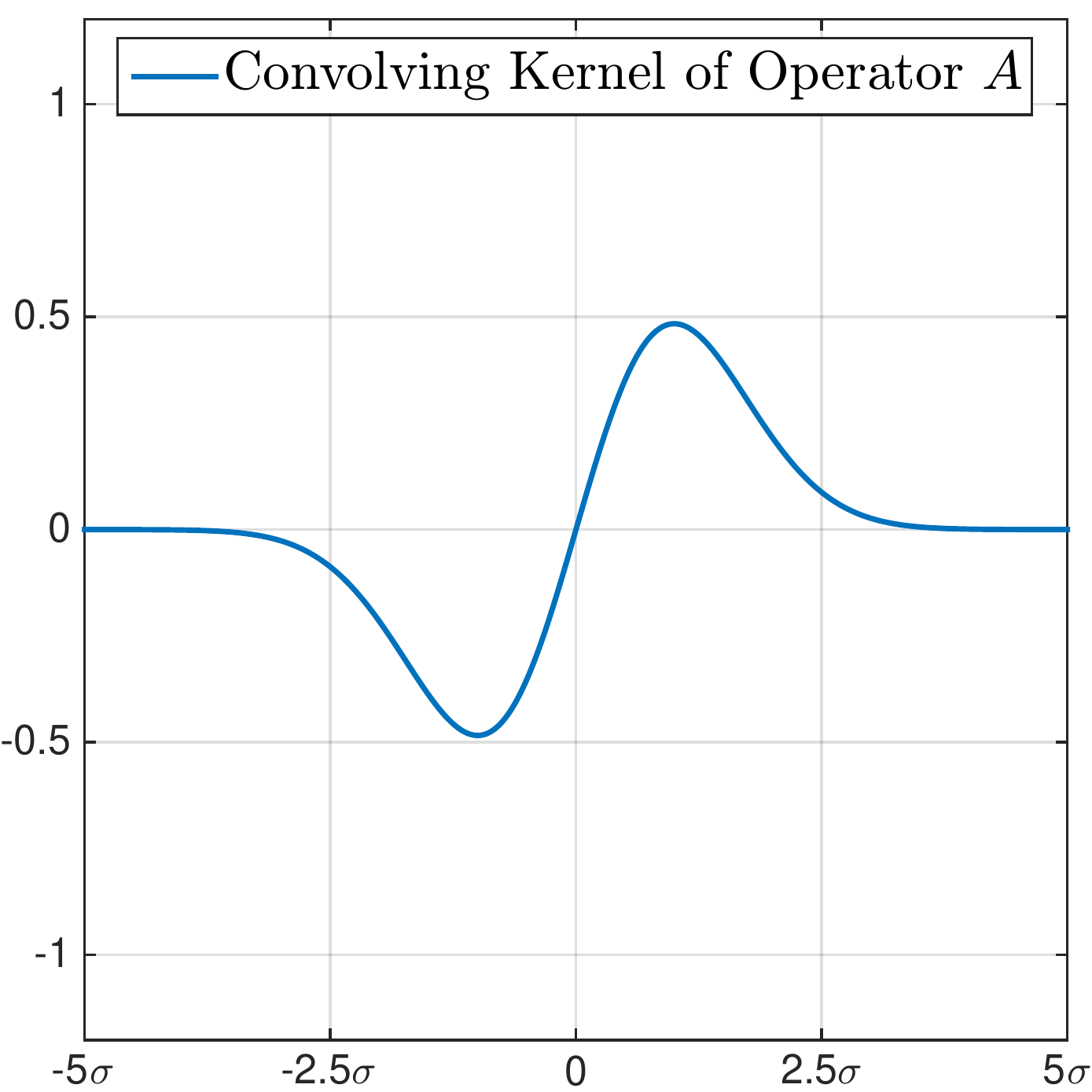}
		}
		\subfigure[][t]{
			\includegraphics[width=0.45\textwidth]{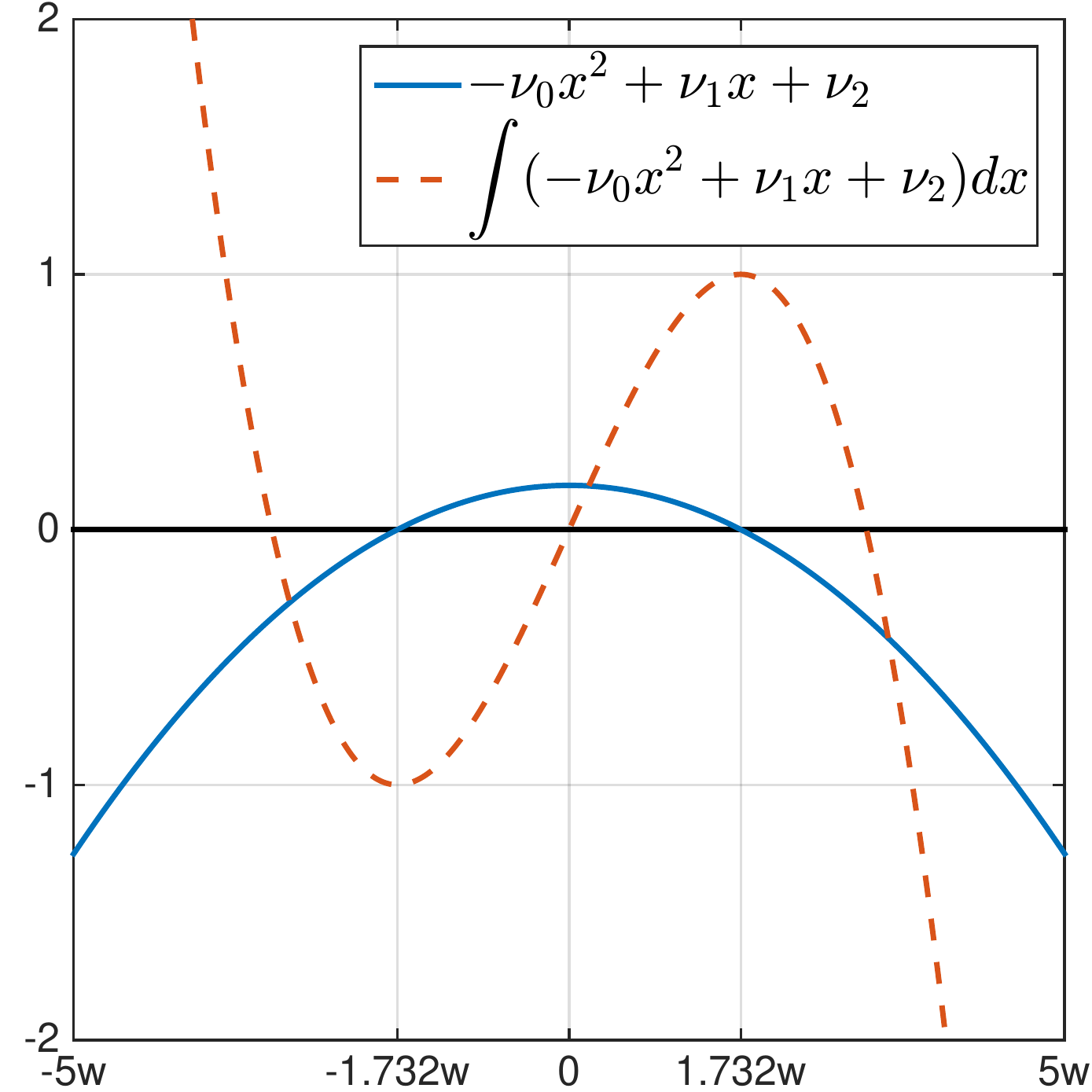}
		}\\
		\subfigure[][t]{
			\includegraphics[width=0.45\textwidth]{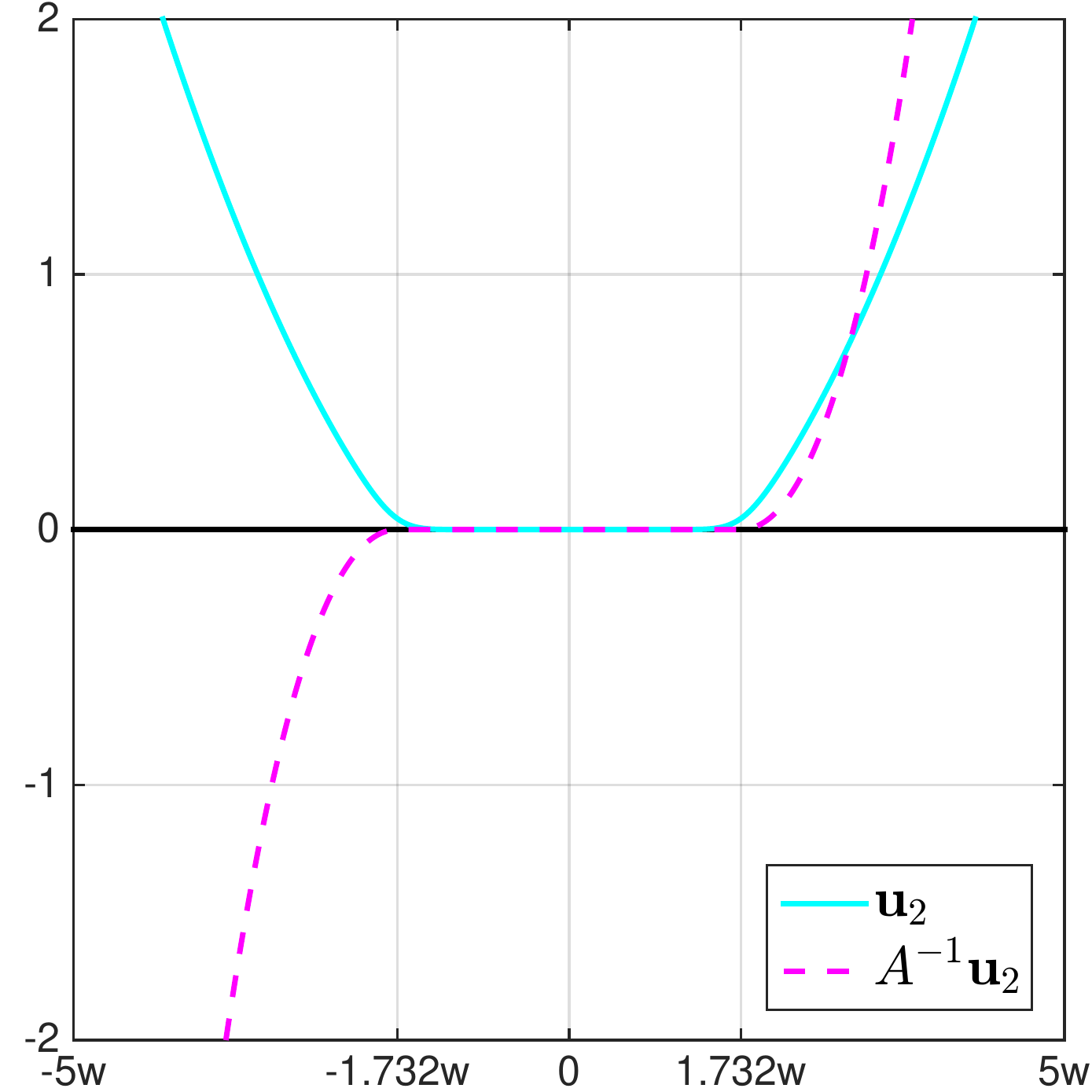}
		}
		\subfigure[][t]{
			\includegraphics[width=0.45\textwidth]{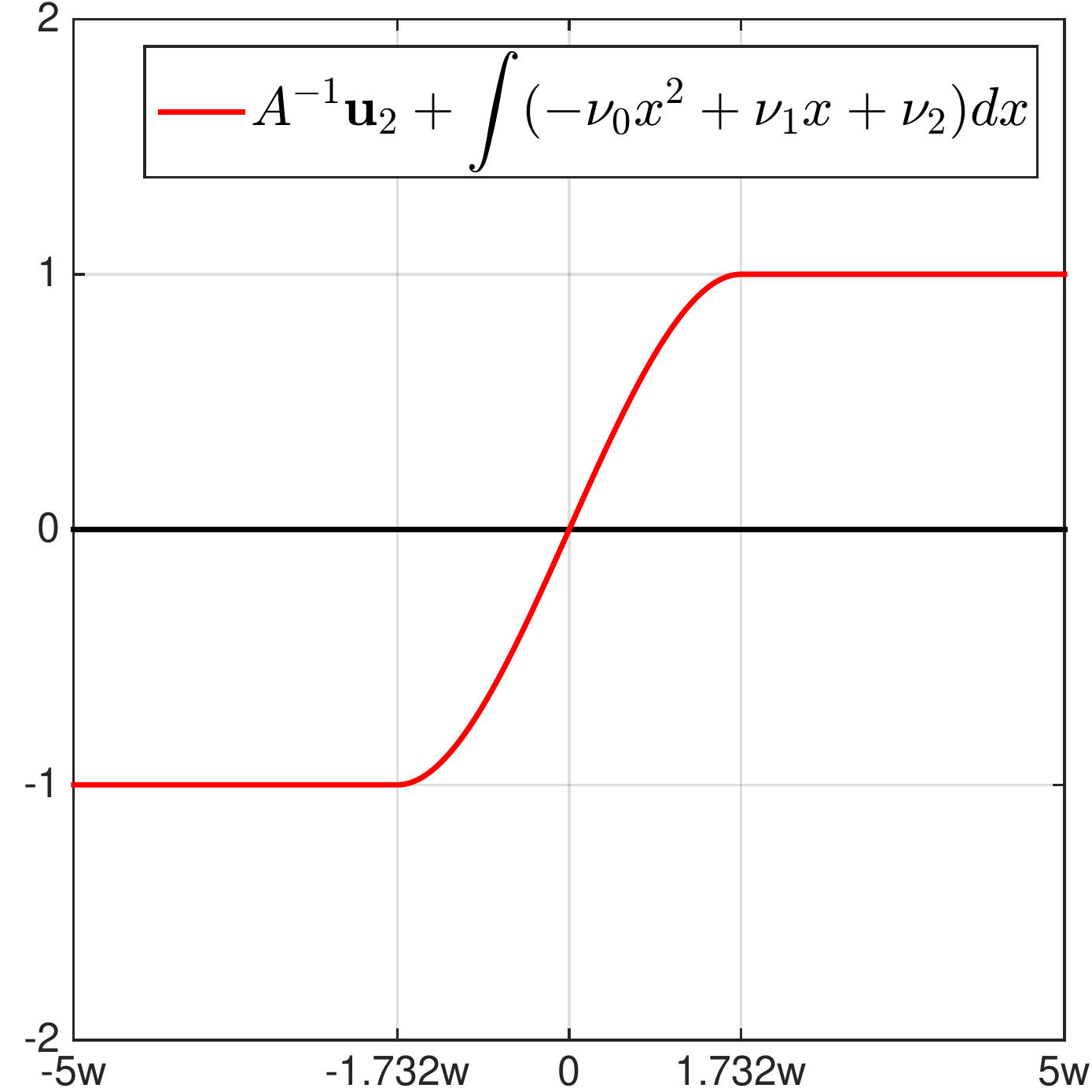}
		}
		\caption{Illustrations of our construction of the dual functions. (a) illustrates the operator $A$, which is essentially a convolution with the shown kernel. (b) shows our constructions of quadratic function $f_{\nu_0,\nu_1,\nu_2}$ and its indefinite integral. (c)  shows our construction of $g = A^{-1} \mb u_2$ and the corresponding nonnegative dual function $\mb u_2$. (d) illustrates the that the the $\ell_\infty$-norm constraint is satisfied.}\label{fig:dualfunctions}
	\end{figure}
	
	From Figure~\ref{fig:dualfunctions}(a), we can see that the linear operator $A$ is closely related to the differentiation operator. Correspondingly, $A^{-1}$ is closely related to indefinite integral operator. Using the moment properties of the half-normal distribution and simple calculus, we derive a few properties about $A$ and $A^{-1}$ when applied to polynomials (see the derivation in the next section).
	\begin{align*}
	&A \mb 1 = \mb 0, &&  \\
	&A\mb x = 2\sigma^2 \mb 1, &&  A^{-1} \mb 1 = \frac{1}{2\sigma^2} \mb x  + C,\\
	&A\mb x^2 = 4\sigma^2 \mb x, && A^{-1} \mb x = \frac{1}{4\sigma^2} \mb x^2 + C, \\
	&A\mb x^3 = 6\sigma^2 \mb x^2  + 6\sigma^4 \mb 1,&& A^{-1} \mb x^2 = \frac{1}{6\sigma^2}\mb x^3 - \frac{1}{2}\mb x + C.
	\end{align*}
	where $C$ is an arbitrary constant. It follows that 
	\begin{align}
	& A^{-1}(-u_1\mb x^2 + \mb u_2 + v_1\mb 1 -v_2 \mb x )\nonumber\\
	=& -\frac{1}{6\sigma^2}u_1 \mb x^3  +  \frac{- 1}{4\sigma^2}v_2 \mb x^2 + \left(\frac{1}{2}u_1 +\frac{1}{2\sigma^2} v_1\right)\mb x + A^{-1} \mb u_2 + C. \label{eq:subgradient}\\
	=& \int ( - \nu_0 \mb x^2 + \nu_1 \mb x + \nu_2) dx + A^{-1} \mb u_2 + C.\nonumber
	\end{align}
	where $\nu_0 = \frac{1}{2\sigma^2}u_1, \nu_1 = -\frac{1}{2\sigma^2}v_2, \nu_2 = \frac{1}{2}u_1 +\frac{1}{2\sigma^2} v_1$. The only restriction of $\nu_0$ is non-negativity, and $\nu_1$ and $\nu_2$ can be arbitrary due to the flexibility of of $v_1$ and $v_2$.

	Let $f_{\nu_0,\nu_1,\nu_2}(x) = - \nu_0 x^2 + \nu_1 x + \nu_2$.
	Take $\nu_0,\nu_1,\nu_2$ such that $f_{\nu_0,\nu_1,\nu_2}(x) \geq 0$ between $[-\sqrt{3}w,\sqrt{3}w]$, and  
	$$\int_{-\sqrt{3}w}^{\sqrt{3}w} ( - \nu_0 x^2 + \nu_1 x + \nu_2) dx = 2.$$
	The coefficients that satisfy these constraints are
	$$
	\nu_0 = \frac{1}{2\sqrt{3}w^3}, \quad \nu_1 = 0, \quad\nu_2 = \frac{\sqrt{3}}{2w},
	$$
	which correspond to 
	\begin{equation}\label{eq:u1_v2_v1}
	u_1 = \frac{\sigma^2}{\sqrt{3} w^3} ,\quad v_2 = 0,\quad v_1 = \frac{\sqrt{3} \sigma^3}{w}  - \frac{\sigma^4}{\sqrt{3} w^3}.
	\end{equation}
	Check that these are feasible in \eqref{eq:dual_opt_over_p}.
	
	Moreover,
	$$F_{\nu_0,\nu_1,\nu_2}(x):=\int f_{\nu_0,\nu_1,\nu_2}(x) dx= -\frac{x^3}{6\sqrt{3}w^3} +\frac{\sqrt{3} x}{2w}.$$	
	Define function $g$, where 
	$$g(x) = \begin{cases}
	F_{\nu_0,\nu_1,\nu_2}(x) +1& \text{ when } x \leq -\sqrt{3}w\\
	F_{\nu_0,\nu_1,\nu_2}(x) -1 & \text{ when } x \geq \sqrt{3}w\\
	0 & \text{otherwise.}
	\end{cases}
	$$
	$g$ is a monotonically increasing function, therefore 
	taking $\mb u_2 =  A g$ obeys $\mb u_2 \geq 0$. Check that
	$$\int f_{\nu_0,\nu_1,\nu_2}(x) dx +  A^{-1} \mb u_2  = - \frac{x^3}{6\sqrt{3}} + \frac{\sqrt{3}x}{2w} + g(x)  =\begin{cases}
	-1 & \text{ when }x \leq -\sqrt{3}w\\
	- \frac{x^3}{6\sqrt{3}w^3} + \frac{\sqrt{3}x}{2w}& \text{when} -\sqrt{3}w\leq x \leq \sqrt{3}w \\
	1& \text{when } x\geq \sqrt{3}w
	\end{cases},$$
	therefore obeys the first constraint in \eqref{eq:dual_opt_over_p}. Together with \eqref{eq:u1_v2_v1}, we form $(u_1,\mb u_2, v_1,v_2)$ which is a feasible dual solution and the primal optimal solution $p^*$ obeys
	$$\|Dp^*\|_1\geq \frac{2\sigma^2}{\sqrt{3}w}- \frac{\sigma^4}{\sqrt{3}w^3}. $$
	
	This bound is sharp when $w\gg \sigma$, but becomes meaningless when $w^2<\sigma^2$. We note that $\|Dp^*\|_1$ is a monotonically decreasing function in $w^2$, therefore the case when $w^2=\sigma^2$ gives a lower bound for the case when $w^2\leq \sigma^2$, therefore for any $w$, we can write	
	$$\|Ap^*\|_1\geq \begin{cases}
	\frac{2\sigma^2}{\sqrt{3}w}- \frac{\sigma^4}{\sqrt{3}w^3}&	\text{ when }w^2 \geq \sigma^2,\\
	\frac{\sigma}{\sqrt{3}}&	\text{ when }w^2 < \sigma^2.
	\end{cases}$$
	Combine with \eqref{eq:Pe_as_1norm}, we get our first claim.

	Now we move on to work on the second claim where $\sigma^2/w \rightarrow 0$. This is equivalent to solving the problem when $w$ is fixed and $\sigma \rightarrow 0$, since we can rescale the real line accordingly. As $\sigma \rightarrow 0$, $\frac{A}{2\sigma^2}$ converges to $\frac{\partial (\cdot)}{\partial x}$. 
	We divide the objective of \eqref{eq:dual_opt_over_p} by $2\sigma^2$. 
	At the limit, the KKT condition of \eqref{eq:dual_opt_over_p} becomes
	$$
	\left\{
	\begin{aligned}
	&\int(- u_1 \mb x^2 + \mb u_2  + v_1 \mb 1 - v_2 \mb x) dx + C \in \partial\|\cdot\|_1 (\partial_x \mb p),\\
	&u_1\geq 0, \mb u_2\geq 0,\\
	&\mb p \text{ is a zero-mean distribution,}\\
	&u_1 (\langle \mb x^2, \mb p\rangle - w^2) = 0,\\
	&\mb p(x) \mb u_2(x) = 0 \text{  for every }x\in \R,
	\end{aligned}
	\right\}
	$$
	where the subgradient of the $\ell_1$-norm is
	$$
	\partial\|\cdot\|_1 (\partial_x \mb p) = \begin{cases}  -1 & \text{when } (\partial_x \mb p)(x) < 0\\
	1 & \text{when } (\partial_x \mb p)(x) > 0\\
	[-1,1] &\text{Otherwise.}
	\end{cases}
	$$
	Now we will construct a set of dual variables $(u_1,\mb u_2, v_1,v_2)$ so that they satisfy the KKT condition with 
	$$\mb p(x)  = \begin{cases} \frac{1}{2\sqrt{3}w} & \text{ when } x \in [-\sqrt{3}w,\sqrt{3}w]\\
	0 & \text{ otherwise.}\end{cases}$$
	First of all, $p(x)$ is a valid zero-mean distribution and $\langle \mb x^2, p\rangle = w^2$.
	$$
	\partial_x \mb p (x)  = \begin{cases} 
	-\infty & \text{ when } x = -\sqrt{3}w \\
	+\infty & \text{ when } x =\sqrt{3}w\\
	0 & \text{ otherwise.}\end{cases}
	$$

	Now consider the range $x\in [-\sqrt{3}w,\sqrt{3}w ]$, where $\mb u_2(x)=0$. $f_{u_1,v_1,v_2}(x) =-u_1 \mb x^2 + v_1 \mb 1 - v_2 \mb x$ is the standard form of a quadratic function, and by $u_1\geq 0$, this is a concave quadratic function. As we did earlier, we choose the parameter of this quadratic function such that 
	\begin{equation*}
	\left\{\begin{aligned}
	&f_{u_1,v_1,v_2}(-\sqrt{3}w)=0,\\
	&f_{u_1,v_1,v_2}(\sqrt{3}w)=0,\\
	& \int_{-\sqrt{3}w}^{\sqrt{3}w}(- u_1 \mb x^2   + v_1 \mb 1 - v_2 \mb x) dx = 2.
	\end{aligned}\right.
	\end{equation*}
	This is always feasible because as $u_1$ goes from $0$ to $\infty$, the area under the curve also continuously and monotonically increases to $\infty$.  Now, choosing $C=1$ ensures that we have $-1\leq (\int f_{u_1,v_1,v_2}(x)dx  +C) \leq 1$, $f_{u_1,v_1,v_2}(-\sqrt{3}w)=-1$ and $f_{u_1,v_1,v_2}(\sqrt{3}w)=1$.
	
	When we have anything outside $[-\sqrt{3}w,\sqrt{3}w ]$, $f_{u_1,v_1,v_2}(x)\leq 0$ and taking $u_2(x) = - f_{u_1,v_1,v_2}(x)$ allows function $\mb u_2 + f_{u_1,v_1,v_2}$ to stay at $0$, which checks he stationarity condition. Therefore, the given dual variables certify that the proposed uniform $\mb p$ is optimal. 
	The objective value $\lim_{\frac{\sigma^2}{w}\rightarrow 0}\frac{w}{\sigma^2}\int |\E_t t(p(x+t )-p(x-t))| dx  = \frac{2}{\sqrt{3}}$. The proof is complete by substituting the quantity into \eqref{eq:Pe_as_1norm} (divide by $2$).
\end{proof}

\section{Derivation of the simple properties of $A$ and $A^{-1}$}
The raw moments of the half-normal distributions are:
\begin{align*}
\mu_1 = \frac{\sqrt{2}\sigma}{\sqrt{\pi}}, &&
\mu_2 = \sigma^2,&&
\mu_3 = \frac{2\sqrt{2}\sigma^3}{\sqrt{\pi}},&&
\mu_4 = 3\sigma^4.
\end{align*}
Let $q$ be the half normal density. We start with the forward operator $A$ on polynomials.
\begin{align*}
&A\mb 1 &=& \int_0^{\infty} |t| q(t) dt  - \int_{-\infty}^0 |t| q(t) dt = 0.\\
&A\mb x &=&  \int_0^{\infty} (x+t)|t| q(t) dt  - \int_{-\infty}^0 (x-t) |t| q(t) dt = 2\int_0^\infty t^2q(t)dt =2 \mu_2.\\
&A\mb x^2 &=&  \int_0^{\infty} (x+t)^2|t| q(t) dt  - \int_{-\infty}^0 (x-t)^2 |t| q(t) dt\\
&&=&\int_0^{\infty} (x^2+2xt + t^2)|t| q(t) dt - \int_0^{\infty} (x^2-2xt + t^2)|t| q(t) dt \\
&&=& 4x \int_0^\infty t^2q(t) = 4x\mu_2.\\
&A\mb x^3 &=&  \int_0^{\infty} (x+t)^3|t| q(t) dt  - \int_{-\infty}^0 (x-t)^3 |t| q(t) dt\\
&&=&\int_0^{\infty} (x^3+3x^2t + 3xt^2+ t^3)|t| q(t) dt - \int_0^{\infty} (x^3-3x^2t + 3xt^2- t^3)|t| q(t) dt \\
&&=& 6x^2 \int_0^\infty t^2q(t) dt + 2\int_0^\infty t^3q(t) dt = 6x\mu_2 + 2\mu_4.
\end{align*}
The inverse operator $A^{-1}$ on $\mb 1,\mb x$ and $\mb x^2$ are obtained by simply applying $A^{-1}$ on both sides and rearrange the terms.

\end{document}